\documentclass[11pt]{article}
\usepackage[top=0.75in,bottom= 0.75in,left= 0.75in,right=0.75in]{geometry}
\usepackage{times}
\usepackage{authblk}

\usepackage{graphicx} 
\usepackage{subfigure} 

\usepackage{amssymb,amsmath,appendix}
\usepackage{graphicx}
\usepackage{subfigure} 
\usepackage{epstopdf}
\usepackage{amsfonts}
\usepackage{amsthm}
\usepackage{algorithm}
\usepackage{bbm}
\usepackage{algpseudocode}
\usepackage{xspace}

\usepackage{natbib}
\usepackage{appendix}

\newtheorem{asm}{Assumption}
\newtheorem{lemma}{Lemma}
\newtheorem{theorem}{Theorem}

\newtheorem{definition}{Definition}

\title{Online Stochastic Optimization  under Correlated Bandit Feedback}
\newcommand{\HCT}{\textit{HCT}\xspace}
\newcommand{\HCTiid}{\textit{HCT}-iid\xspace}
\newcommand{\HCTgamma}{\textit{HCT}-$\Gamma$\xspace}
\newcommand{\comp}{\mathsf{c}}
\newcommand{\A}{\mathcal A}
\newcommand{\X}{\mathcal X}

\newcommand{\G}{\mathcal G}

\newcommand{\calP}{\mathcal P}
\newcommand{\calI}{\mathcal I}
\newcommand{\barH}{\overline{H}}

\renewcommand{\L}{\mathcal L}

\newcommand{\B}{\mathcal B}

\newcommand{\F}{\mathcal F}
\newcommand{\E}{\mathbb E}
\newcommand{\calE}{\mathcal E}
\newcommand{\calS}{\mathcal{S}}

\newcommand{\Prob}{\mathbb P}
\newcommand{\I}{\mathbb I}
\newcommand{\N}{\mathcal N}
\newcommand{\M}{\mathcal M}

\newcommand{\T}{\mathcal{T}}

\newcommand{\hmu}{\widehat{\mu}}

\newcommand{\eps}{\varepsilon}

\newcommand{\hDelta}{\widehat{\Delta}}

\author[1]{Mohammad Gheshlaghi Azar \thanks{\texttt{mohammad.azar@northwestern.edu}}}
\author[2]{Alessandro Lazaric \thanks {\texttt{alessandro.lazaric@inria.fr}}}
\author[3]{Emma Brunskill \thanks{\texttt{ebrun@cs.cmu.edu}} }
\affil[1]{Rehabilitation Institute of Chicago, Northwestern University, Chicago}
\affil[2]{Team SequeL, INRIA Nord Europe  }
\affil[3]{School of Computer Science, Carnegie Mellon University, Pittsburgh }
\begin{document}

\maketitle




\begin{abstract}

In this paper we consider the problem of online stochastic optimization of a locally smooth function under bandit feedback. We introduce the high-confidence tree (HCT) algorithm, a novel any-time $\mathcal{X}$-armed bandit algorithm, and derive regret bounds matching the performance of existing state-of-the-art in terms of dependency on number of steps and smoothness factor. The main advantage of HCT is that it handles the challenging case of correlated rewards, whereas existing methods require that the reward-generating process of each arm is an identically and independent distributed (iid) random process. HCT also improves on the state-of-the-art in terms of its memory requirement  as well as requiring a weaker smoothness assumption on the mean-reward function in compare to the previous anytime algorithms. Finally, we discuss how HCT can be applied to the problem of policy search in reinforcement learning and we report preliminary empirical results.


\end{abstract} 


\section{Introduction}\label{s:introduction}
We consider the problem of maximizing the sum of the rewards obtained by sequentially evaluating an unknown function, where 
the function itself may be stochastic. This is known 
as online stochastic  optimization  under bandit feedback or $\mathcal X$-armed bandit, 
since each function evaluation can be viewed as pulling 
one of the arms in a general arm space $\mathcal X$. 
Our objective is to minimize the cumulative regret relative 
to evaluating/executing at each time point the global 
maximum of the function. 
In particular, we focus on the case that the reward 
(function evaluation) of an arm may depend on prior history of 
evaluations and outcomes. This immediately implies 
that the reward, conditioned on its corresponding arm pull, is not an independent and identically 
distributed (iid) random variable, in contrast to the 
 prior work on $\mathcal X$-armed 
bandits \citep{bull2013adaptive-tree,djolonga13high,BubeckMSS11,Srinivas09,cope2009regret,KleinbergSU08,AuerOS07}. 
$\mathcal X$-armed bandit with correlated reward 
is relevant to many real world optimization applications, 
including internet auctions, adaptive routing, and online 
games. As one important example, we show that the problem of policy 
search in a Markov Decision Process (MDP), a popular approach to 
learning in unknown MDPs, can be framed as an instance 
of 
the setting we consider in this paper
 (Sect.~\ref{s:polMDP}). To the best of our knowledge, the algorithm introduced in this paper is the first to guarantee  sub-linear regret 
in continuous state-action-policy space MDPs.  
 
Our approach builds on recent advances in  $\mathcal X$-armed 
bandits for iid settings~\citep{BubeckMSS11,cope2009regret,KleinbergSU08,AuerOS07}. 
Under regularity assumptions on the mean-reward function 
(e.g. Lipschitz-smoothness), these methods  provide formal  guarantees in terms of bounds on the regret, which is proved to scale sub-linearly w.r.t.\ the number of steps $n$. 
To obtain this regret, these methods rely heavily 
on the iid assumption. 
To handle the correlated feedback, we introduce a new anytime 
$\mathcal X$-armed bandit algorithm, called \emph{high confidence tree} (HCT) (Sect.~\ref{s:algorithm}).
Similar to the HOO  algorithm of  \citet{BubeckMSS11},  
\HCT makes use of a covering binary tree for 
exploring the arm space. The tree is constructed incrementally 
in an optimistic fashion, exploring parts of the arm space 
guided by upper bounds on the potential best reward of 
the arms covered within a particular node. 

Our key insight is that to achieve good performance it is 
 only necessary to expand the tree 
by refining an optimistic node when the estimate of 
the mean-reward of that node has 
become sufficiently accurate. This allows us to obtain an accurate
estimate of the return of a particular arm even in the 
non-iid setting, under some mild ergodicity and mixing assumptions (Sect.~\ref{s:preliminaries}). 
Despite 
handling a more general case of correlated feedback, our 
regret bounds matches (Sect.~\ref{ss:regret.iid}) that of 
HOO \citep{BubeckMSS11}  and zooming algorithm \citep{KleinbergSU08}, 
both of which only apply to iid setting,  in terms of dependency on the number of steps $n$ 
 and the near-optimality dimension $d$ (to be defined later). Furthermore, \HCT also requires milder assumptions on the smoothness of the function, which is required to be Lipschitz only w.r.t. the maximum, whereas HOO assumes the mean-reward to be Lipschitz also between any pair of arms close to the maximum. 
 An important part of our proof of this result 
(though we delay this and all proofs to the supplement, due to space 
considerations) is the development of concentration inequalities 
for non-iid episodic random variables. 
In addition to this main result, the 
structure of our HCT approach has a favorable sub-linear space complexity 
of $O(n^{d/(d+2)} (\log n)^{2/(d+2)})$  and a linearithmic 
runtime complexity, making it suitable for scaling to \emph{big 
data} scenarios. These results meet or improve the space and time 
complexity of prior work
designed for iid data (Sect.~\ref{ss:complexity}), and we will demonstrate this benefit 
in simulations (Sect.~\ref{s:numeric}). We also show how our approach can lead to 
finite-sample guarantees for policy search method, and provide 
preliminary simulation results which show the advantage of our 
method in the case of MDPs.

%
\section{Preliminaries}\label{s:preliminaries}
\noindent\textbf{The optimization problem.} Let $\mathcal X$ be a measurable  space of arms. We formalize the optimization problem as an interaction between the learner and the environment. At each time step $t$, the learner pulls an arm $x_t$ in $\mathcal X$ and the environment returns a reward $r_{t}\in[0,1]$ and possibly a context $y_{t}\in \mathcal Y$, with $\mathcal Y$ a measurable space (e.g., the state space of a Markov decision process). Whenever needed, we relate $r_t$ to the arm pulled by using the notation $r_t(x)$.
 The context $y_{t}$ and the reward $r_t$ may depend on the history of all previous rewards, pulls, contexts and the current pull $x_t$.   
For any time step $t>0$, the space of histories $\mathcal H_t:=([0,1] \times  \mathcal X \times \mathcal Y)^t$ is defined as the space of past rewards, arms,and observations (with $\mathcal H_0 = \emptyset$). 
An environment $M$ corresponds to an infinite sequence of time-dependent probability measures $M=( Q_1, Q_2 ,\dots)$, such that each 
%
%
$Q_t:  \mathcal H_{t-1} \times \mathcal X \to \mathcal M( [0,1] \times \mathcal Y )$ is a mapping from the history $\mathcal H_{t-1}$ and the arm space $\mathcal X$ to the space of probability measures on rewards and contexts.
Let $\mathcal Z=([0,1] \times  \mathcal X \times \mathcal Y)$, at each step $t$ we define the random variable $z_t=(r_t, x_t, y_t)\in \mathcal Z$ and we introduce the filtration $\mathcal F_t $ as a   $\sigma$-algebra generated by $(z_1,z_2,\ldots,z_t)$. At each step $t$, the arm $x_t$ 
is $\mathcal F_{t-1}$-measurable since 
it is based on all the information available up to time $t-1$. The pulling strategy of the learner can be expressed as an infinite sequence of measurable mappings $(\psi_1,\psi_2,\dots)$, where $\psi_t:  \mathcal H_{t-1}  \to \mathcal M( \mathcal X)$ maps $\mathcal H_{t-1}$ to the space of probability measures on arms. We refine this general setting with two assumptions on the reward-generating process.

\begin{definition}[Time average reward]
\label{def:ave}
For any $x\in\X$, $ S>0$ and $0< s \leq S$, the time average reward is

\begin{align}
\bar  r_{s\to S}(x) :=  \frac{1}{S-s+1}\sum_{s'=s}^S  r_{s'}(x) .
\end{align}

\end{definition}
We now state  our first assumption  which guarantees that the mean of the process is well defined (ergodicity).
\begin{asm}[Ergodicity]
\label{asm:ergod}
For any $x\in\X$, any $s>0$ and any sequence of prior pulls $(x_1,x_2,\dots,x_{s-1})$,  the process $(z_t)_{t>0}$ is such that the mean-reward function 

\begin{equation*}
f(x) := \lim{_{S\rightarrow \infty}} \E ( \bar  r_{s\to S}(x) | \mathcal F_{s-1})
\end{equation*}

exists.
\end{asm}
This assumption implies that, regardless of the history of prior observations, if arm $x$ is pulled infinitely many times from time $s$, then the time average reward converges  in expectation to a fixed point which only depends on arm $x$ and is independent from the past history. We also 
make the following mixing assumption~\citep[see e.g.,][]{levin2006markov}.
\begin{asm}[Finite mixing time]
\label{asm:mixing}
There exists a constant $\Gamma\geq 0$ (mixing time) such that for any $x\in\X$,  any $S>0$, any $0< s\leq S$ and any sequence of prior pulls $(x_1,x_2,\dots,x_{s-1})$,  the process $(z_t)_{t>0}$ is such that we have that  

\begin{equation}
|\E[  \sum{_{s'=s}^S (r_{s'}(x) - f(x))}  \big| \F_{s-1}]| \leq \Gamma.
\end{equation}

\end{asm}
This assumption implies that the stochastic reward process induced by pulling arm $x$ can not substantially deviate from $f(x)$ in expectation for more than $\Gamma$ \emph{transient} steps. Note that both assumptions trivially hold if each arm is an iid process: 
in this case $f(x)$ is the  mean of $x$ and $\Gamma=0$.

Given the mean-reward $f$, we assume that the maximizer $x^*\!=\!\arg\max_{x} f(x)$ exists and we denote the corresponding maximum $f(x^*)$ by $f^*$. We measure the performance of the learner over $n$ steps by its regret $R_n$ w.r.t. the  $f^*$, defined as

\begin{equation*}\label{eq:regret}
R_n := nf^*-\sum_{t=1}^{n}r_t.
\end{equation*}

 The goal of learner, at every $0\leq t\leq n$, is to choose a strategy $\psi_t$  such that the regret $R_n$ is as small as possible.

\textbf{Relationship to other models.} Although the learner observes a context $y_t$ at each time $t$, this problem differs from the contextual bandit setting ~\citep[see e.g., the extensions of the zooming algorithm to 
contextual bandits by][]{slivkins2009contextual}. 
In contextual bandits, the context $y\in\mathcal Y$ is provided 
before selecting an arm $x$, and the immediate reward $r_t$ is defined 
to be a function only of the selected arm and input context, 
$r_t(x,y)$. The contextual bandit objective is typically to minimize the regret against 
the optimal arm in the context provided at each step, $y_t$, 
i.e. $x^*_t = \arg\max r_t(x,y_t)$. A key difference is that 
in our model the reward, and next context, may depend on the 
entire history of rewards, arms pulled, and contexts, instead of only 
the current context and arm, and we define $f(x)$ only as the 
average reward obtained by pulling arm $x$. 
In this sense, our model is related to the reinforcement learning (RL) problem 
of trying to find a policy that maximizes the long run reward (see further discussion in Sect.~\ref{s:polMDP}).
Among prior work in RL our setting is most similar to the general reinforcement learning model of \citet{lattimore2013sample} which also considers an arbitrary temporal dependence between the rewards and observations. Our setting differs  from that of \citet{lattimore2013sample}, since we consider the regret in undiscounted reward  scenario, whereas \citet{lattimore2013sample} focus on proving PAC-bounds in discounted reward case. Another difference is that in our model \citep[unlike][] {lattimore2013sample} the space of observations and actions is not needed to be finite. 


\textbf{The cover tree.} 
Similar to recent optimization methods~\citep[e.g.,][]{BubeckMSS11}, 
our approach seeks to minimize regret by smartly building 
an estimate of $f$ using an 
infinite binary \emph{covering tree} $\T$, 
in which each node covers a subset of $\X$.\footnote{The reader is referred to~\citet{BubeckMSS11} for a more detailed description of the covering tree.} We denote by $(h,i)$ the node at depth $h$ and index $i$ among the nodes at the same depth (e.g., the root node which covers $\X$ is indexed by $(0,1)$). By convention $(h+1,2i-1)$ and $(h+1,2i)$ refer to the two children of the node $(h,i)$. The area corresponding to each node $(h,i)$ is denoted by $\calP_{h,i}\subset \X $. These regions must be 
measurable and, at each depth, they partition $\X$ with no overlap:
%
%

\begin{align*}
 \mathcal P_{0,1}&=\mathcal X
 \\
\mathcal P_{h,i}&=\mathcal P_{h+1,2i-1}\cup\mathcal P_{h,2i}&& \forall h\geq0 \text{ and } 1 \leq i\leq2^h.
\end{align*}

%
For each node $(h,i)$, we define an arm $x_{h,i}\in \mathcal P_{h,i}$, 
which the algorithm pulls whenever the node $(h,i)$ is selected.

We now state a few additional geometrical assumptions. 
%
\begin{asm}[Dissimilarity]
\label{asm:dissim}
The space $\X$ is equipped with a dissimilarity function $\ell:\X^2\to\mathbb R$ such that $\ell(x,x')\geq 0$ for all $(x,x')\in \X^2$ and $\ell(x,x)=0$.
\end{asm}
Given a dissimilarity $\ell$, the diameter of a subset $A\subseteq \X$ is defined as $\text{diam}(A):= \sup_{x,y\in A}\ell(x,y)$, while an $\ell$--open ball of radius $\eps>0$ and center $x\in \X$ is defined as $\B(x, \eps):= \{x'\in \X: \ell(x,x')\leq \varepsilon \}$. 

\begin{asm}[local smoothness]
\label{asm:Lip}
We assume that there exist constants $\nu_2, \nu_1>0$ and $0<\rho<1$ such that for all nodes $(h,i)$:

\begin{itemize}
\item[(a)] 
$\text{diam}(\mathcal P_{h,i})\leq \nu_1 \rho^h$ 
\item[(b)] $\exists$  $x^o_{h,i}\in \mathcal P_{h,i}$ s.t. 
$
\mathcal B_{h,i}:= \mathcal B(x^o_{h,i},\nu_2\rho^h) \subset \mathcal P_{h,i},
$
\item[(c)] $\mathcal B_{h,i} \cap \mathcal B_{h,j}=\emptyset$,
\item [(d)] For all $x \in\X$,  
$f^*-f(x)\leq \ell(x^*,x)$.
\end{itemize}

\end{asm}

\textbf{Local smoothness.} These assumptions coincide with those in~\citep{BubeckMSS11}, except for the  local smoothness (Assumption~\ref{asm:Lip}.d), which is weaker than that of~\citet{BubeckMSS11}, where the function is assumed to be Lipschitz between any two arms $x,x'$ \textit{close} to the maximum $x^*$ (i.e., $|f(x)-f(x')| \leq \ell(x,x')$), while here we only require the function to be Lipschitz w.r.t. the maximum.

Finally, we characterize the \textit{complexity} of the problem using the near-optimality dimension, which defines how \textit{large} is the set of $\epsilon$-optimal arms in $\X$. For the sake of 
clarity, we consider a slightly simplified definition of near-optimality dimension w.r.t.~\citet{BubeckMSS11}. 

\begin{asm}[Near-optimality dimension]
\label{asm:near.optimal.dim}
Let $\epsilon = 3\nu_1\rho^h$ and $\epsilon' = \nu_2\rho^h < \epsilon$, for any subset of $\epsilon$-optimal nodes $\X_{\epsilon}=\{x\in\X: f^*-f(x)\leq \epsilon\}$, there exists a constant $C$ such that
%
$\N\big(\X_\epsilon, \ell, \epsilon'\big) \leq C (\epsilon')^{-d}$,
%
where $d$ is the near-optimality dimension of function $f$ and $\N(\X_\epsilon, \ell, \epsilon')$ is the $\epsilon'$-cover number of the set $\X_{\epsilon}$ w.r.t. the dissimilarity measure $\ell$.\footnote{Note that, in many cases,  the near-optimality dimension $d$ can be  much smaller than $D$, the actual  dimension  of arm space $\mathcal X$ in the continuous case. In fact one can show that under some mild   smoothness assumption the near optimality dimension of a  function equals  $0$, regardless of the  dimension of its input  space $\mathcal X$~\citep[see ][for a detailed discussion]{munos2013bandits,valko2013}.}
\end{asm}
%
%
%
%
%




%
%
\section{The High Confidence Tree algorithm}\label{s:algorithm}
\begin{algorithm}[h!]
\begin{algorithmic}
\Require Parameters $\nu_1>0 $, $\rho \in (0,1)$, $c>0$, tree structure $(\mathcal P_{h,i})_{h\geq 0, 1 \leq i\leq 2^i}$ and confidence $\delta$.
\State \textbf{Initialize} $t=1$, $\T_t=\{(0,1),(1,1),(1,2)\}$, $H(t)=1$, $U_{1,1}(t)=U_{1,2}(t)=+\infty$,
\Loop
\If {$t=t^+$} \Comment{Refresh phase}
\ForAll{$(h,i)\in \T_t$ } 
\State $U_{h,i}(t) \leftarrow \hmu_{h,i}(t) + \nu_1\rho^h + \sqrt{\frac{ c^2  \log(1/\tilde \delta(t^+))}{ T_{h,i}(t)} }$
\EndFor;
\ForAll{$(h,i)\in \T_t$ Backward  from $H(t)$} 
\If {$(h,i)\in \text{leaf}(\T_t)$}
\State $ B_{h,i}(t) \leftarrow U_{h,i}(t) $
\Else
\State  $ B_{h,i}(t) \leftarrow \min\big[ U_{h,i}(t),\!\!\max\limits_{j\in \{2i-1,2i\} }\!\!\!\! B_{h+1,j}(t)\big]$
\EndIf
\EndFor
\EndIf;
\State $\{(h_t,i_t), P_t\}\leftarrow \text{OptTraverse}(\T_t)$
\If {Algorithm \textbf{\HCTiid}}
\State Pull arm $x_{h,i}$ and observe $r_t$
\State $t=t+1$
\ElsIf{Algorithm \textbf{\HCTgamma}}
\State $T_{cur}=T_{h_t,i_t}(t)$
\While{$T_{h_t,i_t}(t)< 2 T_{cur} \textbf{ AND } t < t^+$ } 
\State Pull arm $x_{h,i}$ and observe $r_t$
\State $(h_{t+1},i_{t+1})=(h_t,i_t)$
\State $t=t+1$
\EndWhile
\EndIf
\State Update counter $T_{h_t,i_t}(t)$ and empirical average $\hmu_{h_t,i_t}(t)$
\State $U_{h_t,i_t}(t)  \leftarrow  \hmu_{h_t,i_t}(t) + \nu_1\rho^h + \sqrt{\frac{ c^2    \log( 1/ \tilde \delta(t^+) )  }{ T_{h_t,i_t}(t)}}$
\State $ \text{UpdateB}(\T_t,P_t, (h_t,i_t))$
\State $\tau_h(t) = \frac{c^2 \log(1/\tilde \delta(t^+))}{\nu_1^2} \rho^{-2h_t}$
\If {$T_{h_t,i_t}(t) \geq \tau_{h_t}(t)$ AND $(h_t,i_t)=$leaf$(\T)$  }
\State $\calI_t = \{(h_t+1,2i_t-1), (h_t+1, 2i_t)\}$ 
\State  $\T \leftarrow \T \cup \calI_t$ 
\State $U_{h_t+1,2i_t-1}(t)=U_{h_t+1, 2i_t}(t)=+\infty$
\EndIf
\EndLoop
\end{algorithmic}
\par
\caption{The  \HCT algorithm.}
\label{f:hct.iid.anytime}
\par
\end{algorithm}
%
We now introduce the High Confidence Tree (HCT) algorithm for 
stochastic online optimization under bandit feedback. 
Throughout this discussion, a function evaluation is 
equivalent to the reward received from pulling an 
arm (since an arm corresponds to selecting an 
input to evaluate the function at). 
We first describe the general algorithm framework before 
discussing two particular variants: \HCTiid, designed for the case when rewards of a given arm are iid and \HCTgamma 
which handles the correlated feedback case, where  
the reward from pulling an arm may depend on 
all prior arms pulled and resulting outcomes. 
Alg.~\ref{f:hct.iid.anytime} shows the structure of the algorithm 
for \HCTiid and \HCTgamma, noting the minor modifications 
between the two.  
%
%
%
\begin{algorithm}[ht]
\begin{algorithmic}
\Require Tree $\T$
\State $(h,i) \leftarrow (0,1)$, $P\leftarrow (0,1)$ 
\State $T_{0,1}=\tau_0(t)=1$;
\While{$(h,i) \notin \text{Leaf}(\T) \textbf{ AND } T_{h,i}(t) \geq \tau_h(t)$ } 
\If{$B_{h+1,2i-1} \geq  B_{h+1,2i}$} 
\State $(h,i)\leftarrow(h+1,2i-1)$
\Else
\State $(h,i)\leftarrow (h+1,2i)$
\EndIf
\State $P\leftarrow P \cup \{(h,i)\}$
\EndWhile
\State \Return $(h,i)$ and $P$
\end{algorithmic}
\par
\caption{The \textit{OptTraverse} function.}
\label{f:opttraverse.iid}
\par
\end{algorithm}

\begin{algorithm}[ht]
\begin{algorithmic}
\Require Tree $\T$,  the path $P_t$,  selected node $(h_t,i_t)$
\par
\If {$(h_t,i_t) \in \text{Leaf}(\T) $}
\State $B_{h_t,i_t}(t)=U_{h_t,i_t}(t)$
\Else
\State $B_{h_t,i_t}(t)= \min\big[ U_{h_t,i_t}(t), \max\limits_{j \in \{2i_t-1,2i_t\}} B_{h_t+1,j}(t) \big]$
\EndIf;
\ForAll{$(h,i) \in P_t-(h_t,i_t) $ backward} 
\State $B_{h,i}(t) = \min\big[ U_{h,i}(t), \max\limits_{j \in \{2i-1,2i\}} B_{h+1,j}(t) \big]$
\EndFor
\label{f:opttraverse.fast}
\par
\end{algorithmic}
\caption{The \textit{UpdateB} function.}
\end{algorithm}

\textbf{The general structure.} The \HCT algorithm relies on a binary covering tree $\T$ provided as input used to construct a hierarchical approximation of the mean-reward function $f$. At each node $(h,i)$ of the tree, the algorithm keeps track of some statistics regarding the corresponding arm $x_{h,i}$ associated with the node $(h,i)$. 
These include the empirical 
estimate $\widehat \mu_{h,i}(t)$ of the mean-reward function corresponding for arm $x_{h,i}$ at time step $t$ computed as
%
%

\begin{equation}
\label{eq:def.stats}
\widehat \mu_{h,i}(t):=(1/T_{h,i}(t)) \sum\nolimits_{s=1}^{T_{h,i}(t)} r^s(x_{h,i}),
\end{equation}

%
%
where $T_{h,i}(t)$ is the number of times node $(h,i)$ has been selected in the past and $r^s(x_{h,i})$ denotes the $s$-th reward observed after pulling $x_{h,i}$ (while we previously used $r_t$ to denote the $t$-th sample of the overall process). As explained in Sect.~\ref{s:preliminaries}, although a node is associated to a single arm $x_{h,i}$, it also covers a full portion of the input space $\X$, i.e., the subset $\calP_{h,i}$. Thus, similar to the HOO algorithm~\citep{BubeckMSS11}, \HCT also maintains two upper-bounds, $U_{h,i}$  and $B_{h,i}$, which are meant to bound the mean-reward $f(x)$ of all the arms $x\in\calP_{h,i}$. In particular, for any node $(h,i)$, the upper-bound $U_{h,i}$ is computed directly from the observed reward for pulling $x_{h,i}$ as
%

\begin{equation}
\label{eq:def.stats.up}
U_{h,i}(t):=\hmu_{h,i}(t) + \nu_1\rho^h + \sqrt{c^2    \log( 1/ \tilde \delta(t^+) )/T_{h,i}(t)},
\end{equation}

%
where $t^+ = 2^{\lfloor \log(t) \rfloor+1}$ and $\tilde \delta(t):=\min\{c_1 \delta/t, 1\}$.
Intuitively speaking, the second term is related to the \textit{resolution} of node $(h,i)$ and the third term accounts for the \textit{uncertainty} of $\hmu_{h,i}(t)$ in estimating the mean-reward $f(x_{h,i})$. The $B$-values are designed to have a tighter upper bound on $f(x)$ by taking the minimum between $U_{h,i}$ for the current node, and the maximum upper bound of the node's two child nodes, if present.\footnote{Since the node's children together contain the same input space as the node (i.e., $\mathcal P_{h,i}=\mathcal P_{h+1,2i-1}\cup\mathcal P_{h,2i}$), the node's maximum cannot be greater than the maximum of its children.} More precisely, 
%

\begin{equation}
\label{eq:def.stats.upB}
B_{h,i}(t)\!=\!
\begin{cases}
U_{h,i}(t) &(h,i)\!\in\!\text{leaf}(\T_t)
\\
&
\\
\begin{aligned}
\min[ U_{h,i} (t), \!\!\!\max_{j \in \{2i-1,2i\}} \!\!\!B_{h+1,j}(t)]
 \end{aligned}&\text{otherwise}.
\end{cases}
\end{equation}

%
To identify which arm to pull, the algorithm traverses the tree along a path $P_t$ obtained by selecting nodes with maximum $B_{h,i}$ until it reaches an optimistic node $(h_t,i_t)$, which is either a leaf or a node which is not pulled enough w.r.t. to a given threshold $\tau_h(t)$, i.e., $T_{h,i}(t) \leq \tau_h(t)$ (see function \textsl{OptTraverse} in Alg.~\ref{f:opttraverse.iid}). Then the arm $x_{h_t,i_t}\in\calP_{h_t,i_t}$ corresponding to selected node $(h_t,i_t)$ is pulled. 

The key step of \HCT is in deciding when to expand the tree. 
We expand a leaf node only
 if we have pulled its corresponding arm 
a sufficient number of times such that the uncertainty over 
the maximum value of the arms contained within that node 
is dominated by size of the subset of $\X$ it covers. 
Recall from Equation~\ref{eq:def.stats.up} 
that the  upper bound $U_{h,i}$ of a node $(h,i)$ 
two additional terms added to the empirical average reward. 
The first $\nu_1\rho^h$ is a constant that depends 
only on the node depth, and bounds the possible 
difference in the mean-reward function between the 
representative arm for this node and all other arms 
also contained in this node, i.e., the difference between $f(x_{h,i})$ and $f(x)$ for any other $x\in\calP_{h,i}$ (as follows from 
 Assumptions~\ref{asm:dissim} and~\ref{asm:Lip}). 
The second term depends only on $t$ and decreases with the number of pulls to this node. At some point, 
the second term will become smaller than the 
first term, meaning that the uncertainty over the possible 
rewards of nodes in $\calP_{h,i}$ becomes dominated by the potential 
difference in rewards amongst arms that are contained within the same node. This means that the domain $\calP_{h,i}$ is too large, and thus the resolution of the current approximation of $f$ in that region needs to be increased. Therefore our approach 
chooses the point at which these two terms become
 of the same magnitude to expand a node, 
which occurs when the the number of pulls $T_{h_t,i_t}(t)$ has exceeded a threshold 
%
%

\begin{align}\label{eq:tau}
\tau_h(t) := c^2 \log(1/\tilde\delta(t^+))\rho^{-2h_t}/\nu_1^2.
\end{align}

%
%
(see Sect.~\ref{app:proof1} of the supplement for further discussion). It is 
at this point that expanding the node to 
two children 
can lead to a more accurate approximation of $f(x)$, since $\nu_1\rho^{h+1} \leq \nu_1\rho^{h}$. Therefore if $T_{h_t,i_t}(t)\geq\tau_h(t)$, the algorithm expands the leaf, creates both children leaves, and set their $U$-values to $+\infty$. Furthermore, notice that this expansion only occurs for nodes which are likely to contain $x^*$. In fact, \textsl{OptTraverse} does select nodes with big $B$-value, which in turn receive more pulls and are thus expanded first.
The selected arm $x_{h_t,i_t}$ is pulled either for a single time step (in \HCTiid) or for a full episode (in \HCTgamma), and then the statistics of all the nodes along the optimistic path $P_t$ are updated backwards. The statistics of all the nodes outside the optimistic path remain unchanged.

As \HCT is an anytime algorithm,  we periodically 
need to recalculate the node upper bounds to guarantee their validity with \textit{enough} probability (see supplementary material for a more precise discussion). To do so, at the beginning of each step $t$, 
the algorithm verifies whether the $B$ and $U$ values need to be refreshed or not. In fact, in the definition of $U$ in Eq.~\ref{eq:def.stats.up}, the uncertainty term depends on the confidence $\tilde \delta(t^+)$, which changes at $t=1,2,4,8,\ldots$. Refreshing the $U$ and $B$ values triggers a ``resampling phase'' of the internal nodes of the tree $\T_t$ along the optimistic path. In fact, the second condition in the \textsl{OptTraverse} function (Alg.~\ref{f:opttraverse.iid}) forces \HCT to pull arms that belong to the current optimistic path $P_t$ until the number of pulls $T_{h,i}(t)$ becomes greater than $\tau_{h}(t)$ again. Notice that the choice of the confidence term $\tilde\delta$ is particularly critical. For instance, choosing a more natural $\tilde\delta(t)$ would tend to trigger the refresh (and the resampling) phase too often thus increasing the computational complexity of the algorithm and seriously affecting its theoretical properties in the correlated feedback scenario. \footnote{If we refresh  the upper-bound statistics at every time step the algorithm may select a different arm at every time step, whereas in  correlated feedback scenario having a  small number of switches is critical for the convergence of the algorithm.} On the other hand, the choice of $\tilde\delta(t^+)$ limits the need to refresh the $U$ and $B$ values to only $O(\log(n))$ times over $n$ rounds and guarantees that $U$ and $B$ are valid upper bounds with high probability.

%

 \textbf{\HCTiid and \HCTgamma.}
The main difference between the two implementations of \HCT is that, while \HCTiid pulls the selected arm for only one step before re-traversing the tree from the root to again find another optimistic node, \HCTgamma
pulls the the representative arm of the optimistic node for an episode of $T_{cur}$ steps, where $T_{cur}$ is the number of pulls of arm $x_{h,i}$ at the beginning of episode. In other words, 
the algorithm  doubles the number of pulls  of each arm throughout the episode. Note that not all the episodes may actually finish after $T_{cur}$ steps and double the number of pulls:  The algorithm may interrupt the episode when the confidence bounds of $B$ and $U$ are not valid anymore (i.e., $t\geq t^+$) and perform a refresh phase. 

The reason for this change is that in order to accurately estimate the mean-reward given correlated bandit 
feedback, it is necessary to pull an arm for a series 
of pulls rather than a single pull. Due to our 
assumption on the mixing time (Assumption.~\ref{asm:mixing}), 
pulling an arm for a sufficiently long sequence will 
provide an accurate estimate of the potential mean 
reward even in the correlated setting, thus ensuring that the empirical average rewards $\hmu_{h,i}$ actually concentrates towards 
their mean value (see Lem.~\ref{eq:EgodtoHoff} in the supplementary material). It is this mechanism, 
coupled with only expanding the nodes after obtaining 
a good estimate of their mean reward, that allows us 
to handle correlated feedback setting. 
Although in this sense \HCTgamma is more general, we do however include the \HCTiid variant because 
whenever the rewards are iid it performs better than \HCTgamma. This is due to the fact that, unlike \HCTiid, \HCTgamma has to keep pulling an arm for a full episode even when there is evidence that another arm could be better. We also notice that there is a small difference in the constants  $c_1$  and $c$ between \HCTiid and \HCTgamma: in the case of \HCTiid $c_1:=\sqrt[8]{\rho/(3\nu_1)}$ and $c:=2\sqrt{1/(1-\rho)}$, whereas \HCTgamma  uses $c_1:=\sqrt[9]{\rho/(4\nu_1)}$ and $c:=3(3\Gamma+1)\sqrt{1/(1-\rho)}$.



\section{Theoretical Analysis}\label{s:main}
In this section we analyze the regret and the complexity of \HCT. All the proofs are reported in the supplement.
%
%
\subsection{Regret Analysis}\label{ss:regret.iid}
We start by reporting a bound on the maximum depth of the trees generated by \HCT. 
\begin{lemma}\label{lem:bound.depth.anytime}
Given the threshold $\tau_h(t)$ in Eq.~\ref{eq:tau}, the depth $H(n)$ of the tree $\T_n$ is bounded as
%
%

\begin{align}\label{eq:max.depth}
H(n) \leq H_{\max}(n) = 1/(1-\rho)\log(n\nu_1^2/(2(c\rho)^2)).
\end{align}

%
%
\end{lemma}
This bound guarantees that \HCT never expands trees beyond depth $O(\log n)$. This is ensured by the fact the \HCT waits until the value of a node $f(x_{h,i})$ is sufficiently well estimated before expanding it and this implies that the number of pulls exponentially grows with the depth of tree, thus preventing the depth to grow linearly as in HOO.

We report regret bounds in high probability, bounds in expectation can be obtained using standard techniques. 
\begin{theorem}[Regret bound of \HCTiid]
\label{thm:hct.iid}
Pick a  $\delta\in(0,1)$. 
  Assume that at each step $t$, the reward $r_t$, conditioned on $x_t$, is independent of all prior random events and the immediate mean reward $f(x)=\mathbb E(r|x)$ exists for every $x\in \mathcal X$. Then under Assumptions~\ref{asm:dissim}--\ref{asm:near.optimal.dim} the regret of \HCTiid in $n$ steps is, with probability (w.p.) $1-\delta$,\footnote{Constants are provided in Sect.~\ref{app:proof1} of the supplement.}
%
%

\begin{align*}
R_n &\leq O\big(\big(\log \left(n/\delta\right)\big)^{1/(d+2)}n^{(d+1)/(d+2)}\big).
\end{align*}
%

%
%
  
\end{theorem}
\textbf{Remark (the bound).} We notice that the bound perfectly matches the bound for HOO up to constants (see Thm. 6 in~\citep{BubeckMSS11}). This represents a first sanity check w.r.t.\ the structure of \HCT, since it shows that changing the structure of HOO and expanding nodes only when they are pulled enough, preserves the regret properties of the algorithm. Furthermore, this result holds under milder assumptions than HOO. In fact, Assumption~\ref{asm:Lip}-(d) only requires $f$ to be Lipschitz w.r.t.\ to the maximum $x^*$. Other advantages of \HCTiid are discussed in the Sect.~\ref{ss:complexity} and~\ref{s:numeric}.

Although the proof is mostly based on standard techniques and tools from bandit literature, \HCT has a different structure from HOO (and similar algorithms) and moving from iid to correlated arms calls for the development of a significantly different proof technique. The main technical issue is to show that the empirical average $\hmu_{h,i}$ computed by averaging rewards obtained across different episodes actually converges to $f(x_{h,i})$. In particular, we prove the following high-probability concentration inequality (see Lem.~\ref{eq:EgodtoHoff} in the supplement for further details).
\begin{lemma}\label{eq:EgodtoHoff.text}
Under Assumptions~\ref{asm:ergod} and \ref{asm:mixing}, for any fixed node $(h,i)$ and step $t$, we have that, w.p. $1-\delta$,
%
%

\begin{align*}
| \hmu_{h,i}(t) - f(x_{h,i}) | \leq (3\Gamma+1)\sqrt{\frac{2\log(5/\delta)}{T_{h,i}(t)}}+\frac{\Gamma\log(t)}{T_{h,i}(t)}.
\end{align*}

%
%
Furthermore $K_{h,i}(t)$, the number of episodes in which $(h,i)$ is selected, is bounded by $\log_2(4T_{h,i}(t))+\log_2(t)$.
\end{lemma}
%
This technical lemma is at the basis of the derivation of the following regret bound for \HCTgamma.
\begin{theorem}[Regret bound of \HCTgamma]
\label{thm:hct.corr}
We assume that Assumptions~\ref{asm:ergod}--\ref{asm:near.optimal.dim} hold and that rewards are generated according to the general model defined in Section~\ref{s:preliminaries}. Then the regret of \HCTiid after $n$ steps is, w.p. $1-\delta$,%

\begin{align*}
R_n &\leq O\Big(\big(\log \left(n/\delta\right)\big)^{1/(d+2)}n^{(d+1)/(d+2)}\Big).
\end{align*}

%
%
\end{theorem}
%
\textbf{Remark (the bound).} The most interesting aspect of this bound is that \HCTgamma achieves the same regret as \HCTiid when samples are non-iid. This represents a major step forward w.r.t.\ HOO since it shows that the very general case of correlated arms can be managed as well as the much simpler iid case. In the next section we also discuss how this result can be used in policy search for MDPs.
%
%
%
%
\subsection{Complexity}\label{ss:complexity}

\textbf{Time complexity.}  The run time complexity of both versions of \HCT is $O(n\log(n))$. This is due to the boundedness of the depth $H(n)$ and by the structure of the refresh phase. By Lem.~\ref{lem:bound.depth.anytime}, we have that the maximum depth is $O(\log(n))$. As a result, at each step $t$, the cost of traversing the tree to select a node is at most $O(\log n)$, which also coincides with the cost of updating the $B$ and $U$ values of the nodes in the optimistic path $P_t$. Thus, the total cost of selecting, pulling, and updating nodes is no larger than $O(n\log n)$. Notice that in case of \HCTgamma, once a node is selected is pulled for an entire episode, which further reduces the total selection cost. Another computational cost is represented by the refresh phase where all the nodes in the tree are actually updated. Since the refresh is performed only when $t=t^+$, then the number of times all the nodes are refreshed is of order of $O(\log n)$ and the boundedness of the depth guarantees that the number of nodes to update cannot be larger than $O(2^{\log n})$, which still corresponds to a total cost of $O(n\log n)$. This implies that \HCT achieves the same run time as \emph{T-HOO} \citep{BubeckMSS11}. Though unlike \emph{T-HOO}, our algorithm is fully anytime and it does not suffer from the extra regret incurred due to the truncation and the doubling trick.  

\textbf{Space complexity.}  
The following theorem provides bound on space complexity of the \HCT algorithm.
\begin{theorem}
\label{thm:hct.corr.space}
Under the same conditions of Thm.~\ref{thm:hct.corr}, let $\N_n$ denote the space complexity of \HCTgamma, then we have that
%
%

\begin{equation*}
\mathbb E(\N_n) =O( \log(n)^{2/(d+2)}n^{d/(d+2)} ).
\end{equation*}

%
%
\end{theorem}
The previous theorem guarantees that the space complexity of \HCT scales sub-linearly w.r.t. $n$.  An important observation is that the space complexity of \HCT increases   slower, by   a factor of $\widetilde O (n^{1/(d+2)})$, than its regret. This implies that, for small values of $d$, HCT does not require  to use a  large memory space to achieve a good performance.  An interesting special case is the class of   problem with near-optimality dimension   $d=0$.  For this class of problems the bound translates to a space complexity of $O(\log(n))$, whereas the space complexity of alternative algorithms may be as large as $n$ (see e.g., HOO). As it has been shown in \citep{valko2013} the case of $d=0$ covers  a rather large class of functions, since every function  which satisfies some mild local smoothness assumption, around its global optima, has a near-optimality dimension equal to $0$ \citep[see][for further discussions]{valko2013}. The fact that  \HCT can achieve a near-optimal performance, using only   a relatively small memory  space, which makes it a suitable choice for \emph{big-data} applications, where the algorithms with linear space complexity can not be used due to very large size of the dataset.     

\textbf{Switching frequency.} Finally, we also remark another interesting feature of \HCTgamma. Since an arm is pulled for an entire episode before another arm could be selected, this drastically reduces the number of switches between arms. In many applications, notably in reinforcement learning (see next section), this can be a significant advantage since pulling an arm may correspond to the actual implementation of a complex solution (e.g., a position in a portfolio management problem) and continuously switch between different arms might not be feasible. More formally, since each node has a number of episodes bounded by $O(\log n)$ (Lem.~\ref{eq:EgodtoHoff.text}), then the number of switches can be derived be the number of nodes in Thm.~\ref{thm:hct.corr.space} multiplied by $O(\log n)$, which leads to $O( \log(n)^{(d+4)/(d+2)}n^{d/(d+2)} )$.




\section{Application to Policy Search in MDPs}\label{s:polMDP}
As we discussed in Sect.~\ref{s:preliminaries}, \textit{HCT} is designed to handle the very general case of optimization in problems where there exists a strong correlation among the rewards, arm pulls, and contexts, at different time steps. An important subset of this general class is represented by the problem of policy search in infinite-horizon Markov decision processes.  Notice that the extension to the case of partially observable MDPs is straightforward as long as the POMDP satisfies some ergodicity assumptions. 

A MDP $M$ is defined as a tuple $\langle \calS, \A,  P \rangle$ where $\calS$ is the set of states, $\A$ is the set of actions, $P:\calS\times\A\rightarrow \mathcal M(\calS \times [0,1] )$ is the transition kernel mapping each state-action pair to a distribution over states and rewards.  A (stochastic) policy $\pi:  \calS \to \M(\A)$ is a mapping from states to distribution over actions. 
Policy search algorithms \citep{scherrer2013policy,AzarLB13a,kober2011policy} aim at finding the policy in a given policy set which maximizes the long-term performance. Formally, a policy search algorithm receives as input a set of policies $\G = \{\pi_\theta; \theta\in\Theta\}$, each of them parameterized by a parameter vector $\theta$ in a given set $\Theta \subset \Re^d$. 
Any policy $\pi_{\theta}\in \mathcal G$ induces a state-reward transition kernel $T:\calS\times \Theta \to \mathcal M(\calS\times[0,1])$. $T$ relates to the state-reward-action transition kernel $P$ and the policy kernel $\pi_{\theta}$ as follows

\begin{equation*}
T(ds',dr|s,\theta):=\int _{ u\in \A } P(ds',dr|s,u)\pi_{\theta}(du|s).
\end{equation*}

For any $\pi_{\theta} \in \mathcal G$ and initial state $s_0\in S$, the time-average reward over $n$ steps is 

\begin{equation*}
\mu^{\pi_{\theta}}( s_0, n):= \frac 1n \E \left[ \sum\nolimits_{t=1}^n  r_t\right],
\end{equation*}

where $r_1,r_2,\dots,r_n$ is the sequence of rewards observed by running $\pi_{\theta}$ for $n$ steps staring at $s_0$. If the Markov reward process induced by $\pi_\theta$ is ergodic, $\mu^{\pi_{\theta}}( s_0, n)$ converges to a fixed point independent of the initial state $s_0$. The average reward of $\pi_\theta$ is thus defined as 

\begin{equation*}
\mu(\theta):=\lim_{n\to\infty} \mu^{\pi_{\theta}}(s_0, n).
\end{equation*}

The goal of  policy search  is to find the best $\theta^*= \arg\max_{\theta\in \Theta} \mu(\theta)$. 
    Note that $\pi_{\theta^*}$ is optimal in the policy class $\mathcal G$ and it may not coincide with the optimal policy $\pi^*$ of the MDP, when $\pi^*$  is not covered by $\mathcal G$. 
        
It is straightforward now to match the MDP scenario to the general setting in Sect.~\ref{s:preliminaries}, notably mapping $\Theta$ to $\X$ and $\mu(\theta)$ to $f(x)$. More  precisely the parameter space $\theta\in \Theta$ corresponds to the space of arms $\X$, since in the policy search we want to explore the parameter space $\Theta$ to learn the best parameter $\theta^*$. Also the  state space $\mathcal S$ in MDP setting is the special from of context space of Sect.~\ref{s:preliminaries} where here the contexts evolve according to some controlled Markov process. Further  the transition kernel $T$, which at each  time step $t$ determines the distribution on the current state and reward   given the last state and $\theta$    is again a special case of  of the more general $(Q_t)_t$ which may depend on the entire history of prior observations. Likewise   $\mu(\theta)$, $\mu^*_{\Theta}$ and $\theta^*$ translate into $f(\theta)$, $f^*$ and $x^*$, respectively, using the notation of Sect.~\ref{s:preliminaries}. The Asm. \ref{asm:ergod} and  \ref{asm:mixing} in Sect. \ref{s:preliminaries} are also the general version of   the standard ergodicity and mixing assumption in MDPs, in which  the notion of filtration in assumptions of Sect. \ref{s:preliminaries}  is  simply replaced by the the initial state $s_0\in \mathcal S$. 

This allows us to directly apply \HCTgamma to the problem of policy search. The advantage of \HCTgamma algorithm w.r.t. prior work is that, to the best of our knowledge,   it is the first policy search algorithm which provides finite sample guarantees in the form of regret bounds on the performance loss of policy search in MDPs (see Thm.~\ref{thm:hct.corr}), which guarantee that \HCTgamma suffers from a small sub-linear regret  w.r.t. $\pi_{\theta^*}$. 
Also it is not difficult to prove  that the policy induced by \HCTgamma  has a  small simple regret, that  is, the average reward of the policy chosen by \HCTgamma converges to  $\mu(\theta^*)$ with a polynomial rate.\footnote{Refer to \citet{BubeckMSS11,munos2013bandits}   for  how to transform bounds on accumulated regret  to simple regret bounds.}  Another interesting feature of \HCTgamma is that can be readily used  in large  (continuous) state-action problems since it does not make any restrictive assumption on  the size of  state-action space. 

\textbf{Prior regret bounds for continuous MDPs.} A  related work to \HCTgamma is the UCCRL algorithm by~\citet{ortner2012online}, which extends the original UCRL algorithm~\citep{jaksch2010near-optimal} to continuous state spaces. Although a direct comparison between the two methods is not possible, it is interesting to notice that the assumptions used in UCCRL are stronger than for \HCTgamma, since they require both the dynamics and the reward function to be globally Lipschitz. Furthermore, UCCRL requires the action space to be finite, while \HCTgamma can deal with any continuous policy space. Finally, while \HCTgamma is guaranteed to minimize the regret against the best policy in the policy class $\mathcal G$, UCCRL targets the performance of the actual optimal policy of the MDP at hand. Another relevant work is the  OMDP algorithm of \citet{NIPS2013_Yassin} which deals with the problem of RL in continuous state-action MDPs with adversarial rewards. OMDP achieves a  sub-linear regret under the assumption that the space of policies is finite, whereas in HCT the space of policy can be continuous.


\section{Numerical Results}\label{s:numeric}
\begin{figure*}[htp!]
\begin{center}
\subfigure[]
{\includegraphics[width=0.35\textwidth]{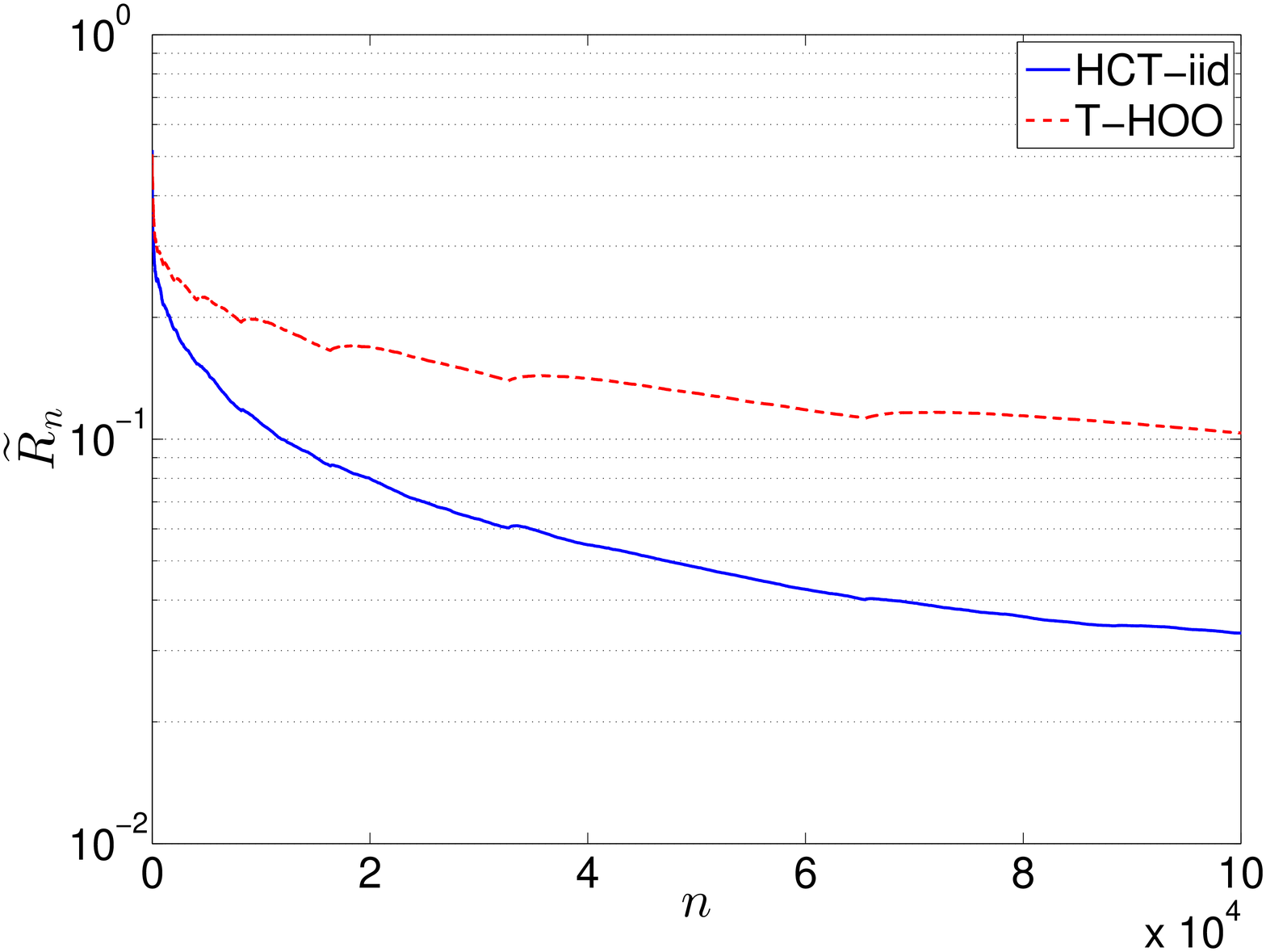}
\label{fig:res.iid.reg}
}
\subfigure[]
{
\includegraphics[width=0.35\textwidth]{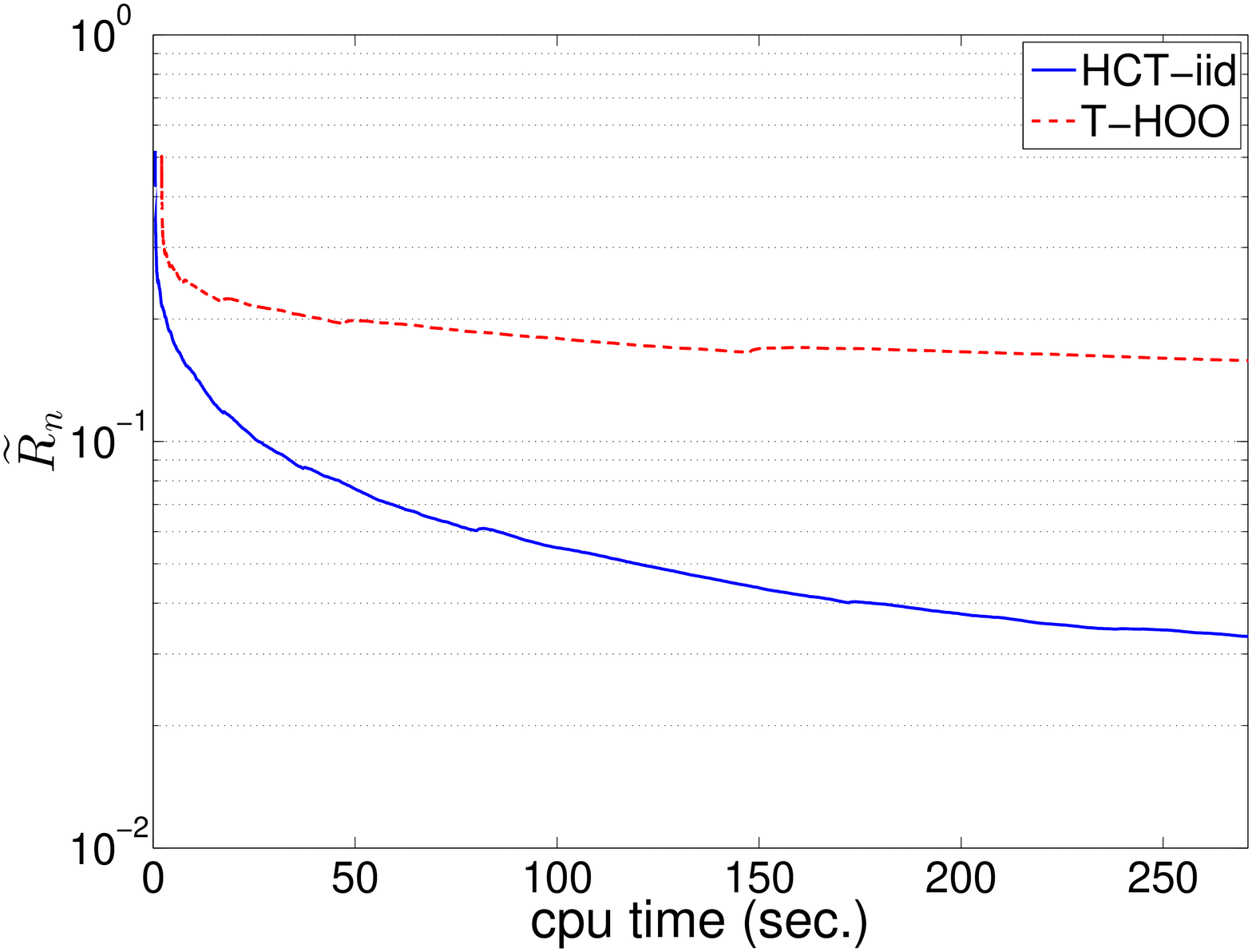}
\label{fig:res.iid.cpu}
}
\subfigure[]
{\includegraphics[width=0.35\textwidth]{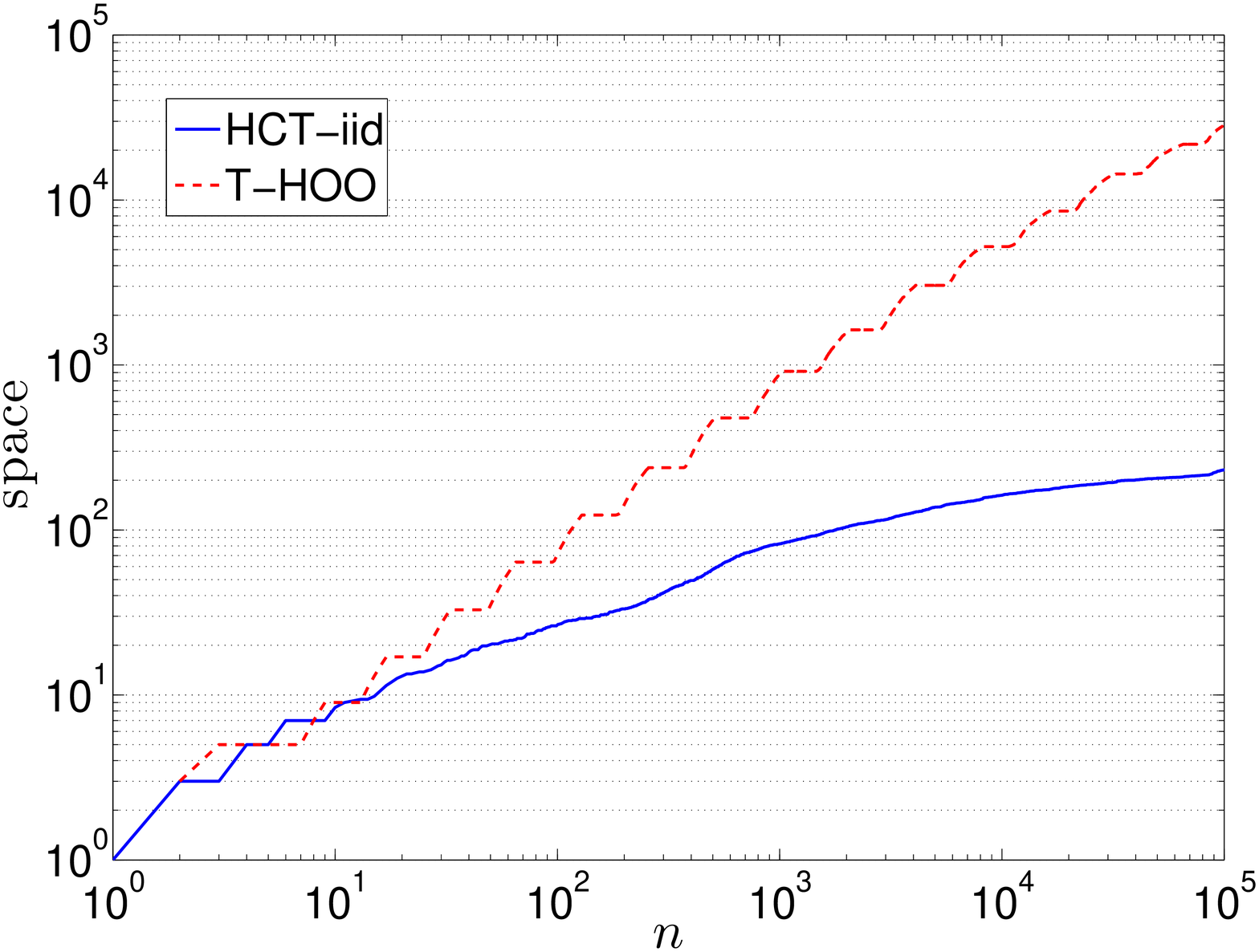}
\label{fig:res.iid.mem}
}
\end{center}
\label{fig:res.iid}
\vspace{-.5cm}
\caption{Comparison of the Performance of $\text{HCT-iid}$ and the Previous Methods under the iid Bandit Feedback.}
\vspace{-.1cm}
\end{figure*}

While our primary contribution is the definition of \HCT and its technical analysis, 
we also give some preliminary 
simulation results to demonstrate some 
of its properties. 
\begin{figure*}[htp]
\begin{center}
\subfigure[]
{\includegraphics[width=0.35\textwidth]{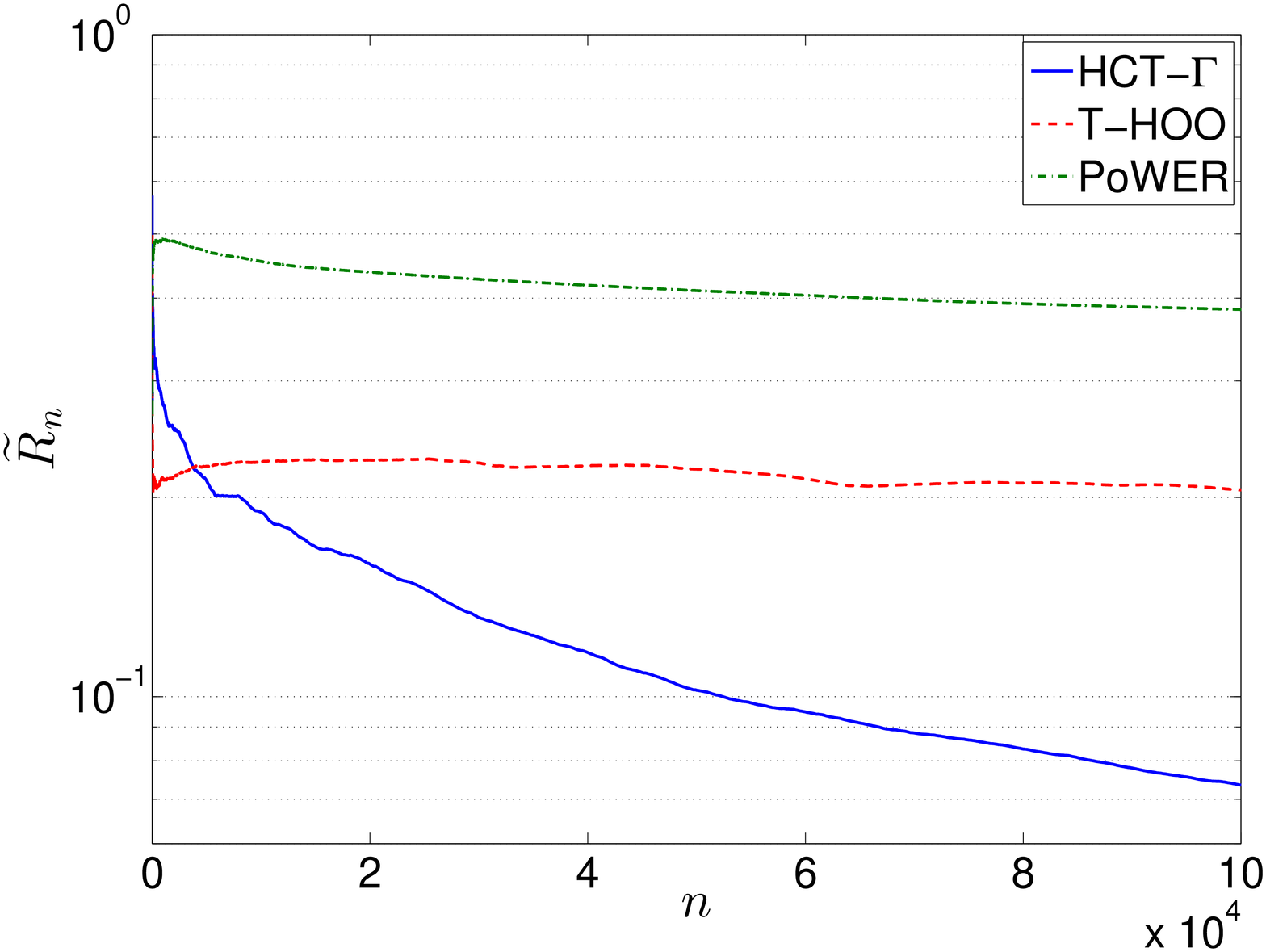}
\label{fig:res.dep.reg}
}
\subfigure[]{
\includegraphics[width=0.35\textwidth]{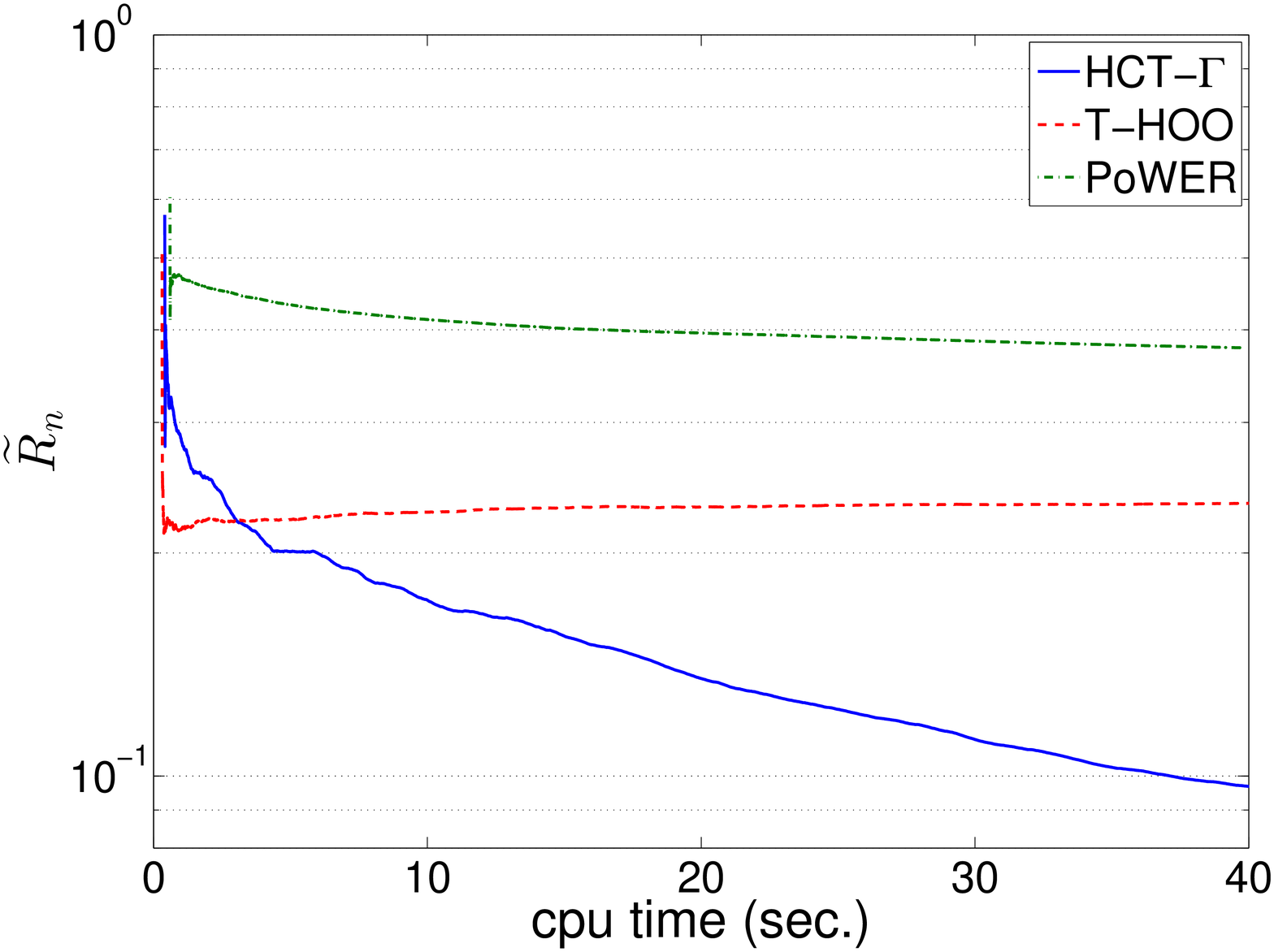}
\label{fig:res.dep.cpu}
}
\subfigure[]{
\includegraphics[width=0.35\textwidth]{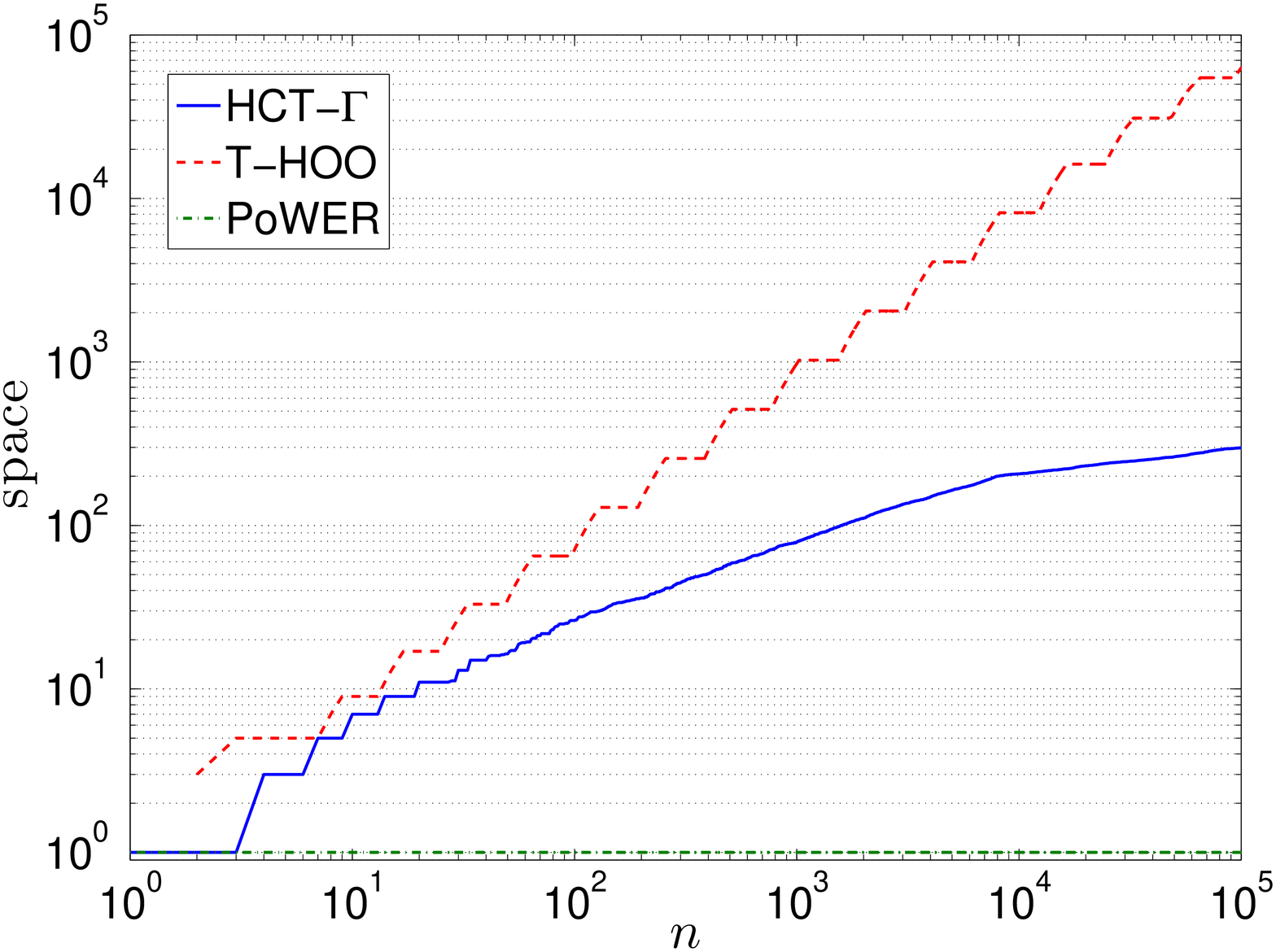}
\label{fig:res.dep.mem}
}
\label{fig:res.dep}
\vspace{-.4cm}
\caption{Comparison of the Performance of $\text{HCT}-\Gamma$ and the Previous Methods under Correlated Bandit Feedback (MDP setting)}
\vspace{-.4cm}
\end{center}
\end{figure*}

\textbf{Setup.} We focus on minimizing 
the regret across repeated noisy evaluations of 
the garland function $f(x)=x(1-x)(4-\sqrt{|\sin(60x)|})$  
relative to repeatedly selecting its global optima.
 We select this function due to its 
 several interesting properties: (1) it contains many local optima, (2) it is locally smooth around its global optima $x^*$ (it behaves as
 $f^*-c|x -x^*|^\alpha$, for  $c =2$  and  $\alpha=1/2$), (3) it is also possible to show that the near-optimality dimension $d$ of $f$ equals $0$.
 
 \begin{figure}[ht]
\begin{center}
\includegraphics[width=0.4\textwidth]{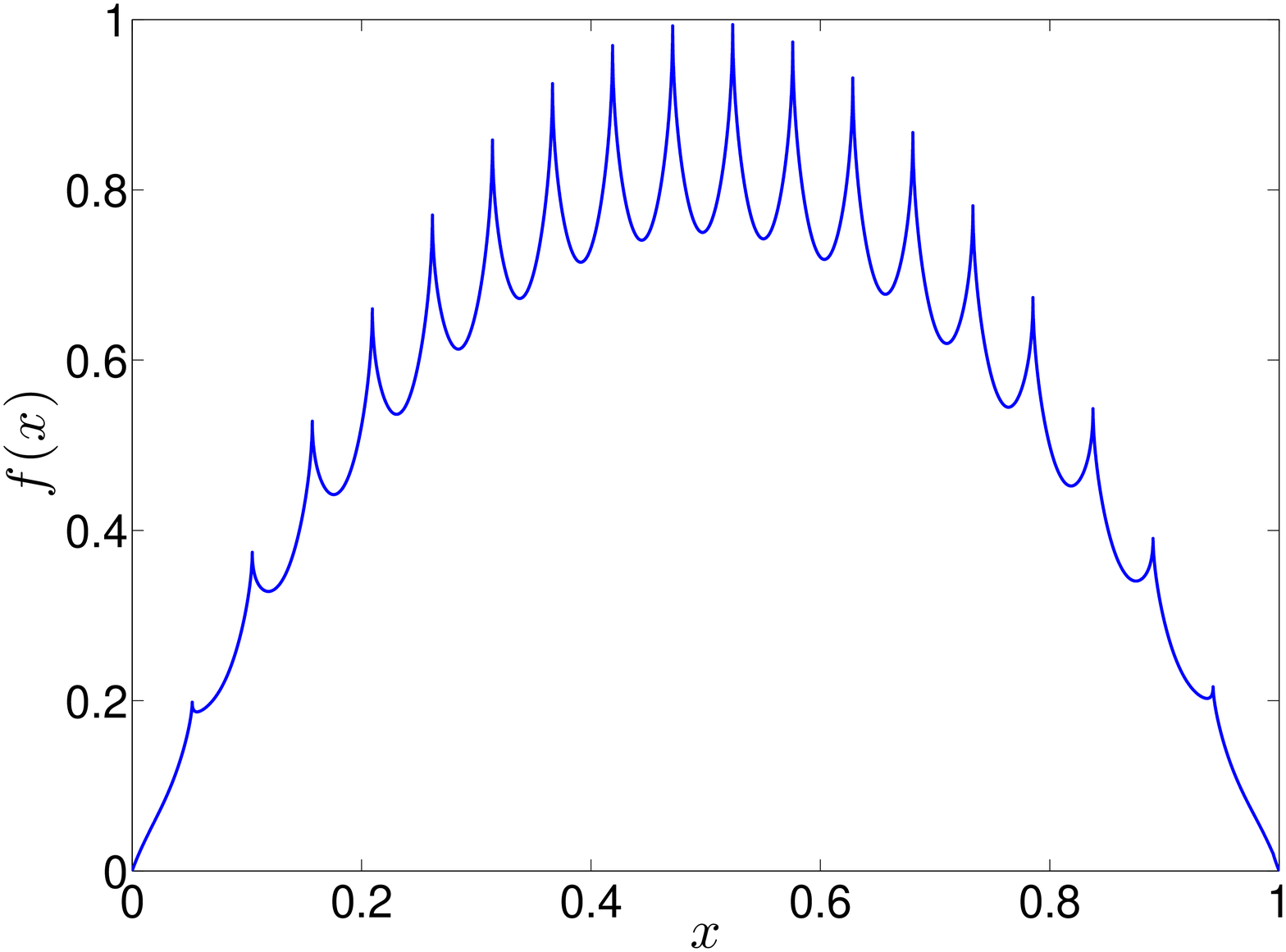}
\caption{The garland function.  }
\label{fig.garland}
\end{center}
\end{figure}

We evaluate the performances of each algorithm in terms of 
the per-step regret, $\widetilde R_n=R_n/n$. Each run is $n=10^5$ 
steps and we average the performance on $10$ runs. For all the algorithms compared in the following, parameters\footnote{For both HCT and T-HOO we introduce a tuning parameter  
used to multiply the upper bounds, while for PoWER we optimize the window for computing the weighted average.} are optimized to maximize their performance.

\textbf{I.i.d.\ setting.}
For our first experiment
%

  we compare \HCTiid to 
the truncated hierarchical optimistic optimization 
(T-HOO) algorithm~\cite{BubeckMSS11}. T-HOO 
is a state-of-the-art $\mathcal X$-armed bandit algorithm,  developed 
as a computationally-efficient alternative of HOO. 
In Fig.~\ref{fig:res.iid.reg} we show the per-step 
regret, the runtime, and the space requirements of 
each approach. As predicted by the theoretical bounds, 
the per-step 
regret  $\widetilde R_n$ of both \HCTiid and truncated \textit{HOO}  decrease rapidly with number of steps.  Though 
the big O theoretical bounds are identical for 
both approaches, empirically we observe in this 
example that  \HCTiid outperforms \textit{T-HOO}  by a large margin.  Similarly, though the computational complexity 
of both approaches matches in the dependence on 
the number of time steps, empirically we observe 
that our approach outperforms \textit{T-HOO} 
(Fig.~\ref{fig:res.iid.cpu}). 
Perhaps the most significant expected 
advantage of HCT-iid over T-HOO for 
iid settings is in the space requirements. 
HCT-iid has a  space requirement for this domain 
that scales logarithmically 
with the time step $n$, as predicted by 
Thm.~\ref{thm:hct.corr.space}. since   
the near-optimality dimension  $d=0$). 
In contrast,  a brief analysis of \textit{T-HOO} 
suggests that its space requirements can grow 
polynomially, and indeed 
in this domain we 
observe  a polynomial growth of memory usage for T-HOO. 
These patterns mean that 
 \HCTiid can achieve a very small regret using a sparse decision tree with only few hundred nodes, whereas truncated \textit{HOO}   requires  orders of magnitude more nodes than \HCTiid.  

\textbf{Correlated setting.}
We create a continuous-state-action  
MDP out of the previously described 
Garland function by introducing the state of the environment 
$s$. Upon taking continuous-valued action $x$, the state of 
the environment changes deterministically to 
$s_{t+1}=(1-\beta) s_t +\beta x$, where we set $\beta=0.2$. 
The agent receives a stochastic reward for being in state $s$, 
which is (the Garland function) $f(s)+\eps$, where as before 
$\eps$ is drawn randomly from $[0,1]$. The initial state $s_0$ is also 
drawn randomly from $[0,1]$. A priori, the agent does not know 
the transition or reward function, making this a reinforcement 
learning problem. Though not a standard benchmark RL instance, 
this problem has multiple local optima and therefore is a interesting 
case for policy search, where $\Theta=\X$ is the policy set (which coincides with the action set in this case).

In this setting, we compare \HCTgamma to a PoWER, a standard 
RL policy search algorithm \cite{kober2011policy} on the above MDP problem 
MDP constructed out of garland function.
PoWER uses an Expectation Maximization approach to optimize the policy 
parameters  and is therefore not guaranteed to find the global optima. 
We also compare our algorithm with T-HOO, though this algorithm is specifically designed for iid setting and one may expect that it may fail to converge to global optima under correlated bandit feedback. 
Fig.~\ref{fig:res.dep.reg} 
shows per-step regret of the $3$ approaches in the MDP.  
Only \HCTgamma succeeds in finding the globally optimal policy, 
as is evident because only 
in the case of \HCTgamma does the average regret tends to converge to zero 
(which is as predicted from Thm.~\ref{thm:hct.corr}). The PoWER method 
finds worse solutions  than both stochastic optimization approaches for 
the same amount of computational time, likely due to using EM which is known 
to be susceptible to local optima. On the other hand, its primary advantage is that it has a 
very small memory requirement. Overall this suggests the benefit of our 
proposed approach to be used for online MDP policy search, since it 
quickly (as a function of samples and runtime) can find a global optima, 
and is, to our knowledge, one of the only policy search methods guaranteed 
to do so. 

\section{Discussion and Future Work}
\label{s.discuss}
%
%
In the current version of \textit{HCT} we assume that the learner has access to the information regarding the smoothness of function $f(x)$ and the mixing time $\Gamma$. In many problems those information  are not available to the learner. In the future it would 
be interesting to build on prior work that handles unknown 
smoothness in iid settings and extend it to correlated feedback. For example, 
 \citet{BubeckSY11} require a stronger global Lipschitz assumption and 
propose an algorithm to estimate the Lipschitz constant. 
Other work on the iid setting include 
\citet{valko2013} and \citet{Munos11}, 
which are limited to the simple regret scenario, but 
who only use the mild local smoothness assumption we define in Asm.~\ref{asm:Lip}, and 
do not require knowledge of 
the dissimilarity measure $\ell$. On the other hand, \citet{slivkins2011multi} and~\citet{bull2013adaptive-tree} study the cumulative regret but consider a different definition of smoothness related to the zooming concept introduced by~\citet{KleinbergSU08}. Finally, we notice that to deal with unknown mixing time, one may rely on data-dependent
tail's inequalities, such as  empirical Bernstein inequality~\citep{tolstikhin2013pac,maurer2009empirical}, replacing the mixing time with the empirical variance of the rewards.

In the future we also wish to explore using HCT to optimize 
other problems that can be modeled using correlated bandit feedback.
For example, HCT may be used for policy search in partially 
observable MDPs~\citep{vlassis2009model,baxter2000reinforcement}, as long as the POMDP is ergodic. 

To conclude, in this paper we introduce a new $\mathcal X$-armed bandit algorithm, called \HCT, for optimization under bandit feedback and prove regret bounds and simulation results for it. 
Our approach improves on existing results 
to handle the important case of correlated bandit feedback. This allows 
HCT to be applied to a broader range of problems than prior 
$\mathcal X$-armed bandit algorithms, such as we demonstrate 
by using it to perform 
policy search for continuous MDPs. 
\appendixpage
\appendix


\section{Proof of Thm.~\ref{thm:hct.iid}}\label{app:proof1}

In this section we report the full proof of the regret bound of \HCTiid. 

We begin by introducing some additional notation, required for the analysis of both algorithms. 
We denote the indicator function of an event $\calE$ by $\mathbb I_{\calE}$. For all $1\leq h\leq H(t)$ and $t>0$, we denote by $\calI_h(t)$ the set of all nodes created by the algorithm at depth $h$ up to time $t$ and by $\calI^+_h(t)$ the subset of $\calI_h(t)$ including only the internal nodes (i.e., nodes that are not leaves), which corresponds to nodes at depth $h$ which have been expanded before time $t$. At each time step $t$, we denote by $(h_t,i_t)$ the node selected by the algorithm. For every $(h,i)\in \mathcal T$, we define the set of time steps when $(h,i)$ has been selected as $\mathcal C_{h,i}:=\{t=1,\ldots,n: (h_t,i_t)=(h,i)\}$. We also define the set of times that a child of $(h,i)$ has been selected as $\mathcal C^c_{h,i}:=\mathcal C_{h+1,2i-1}\bigcup\mathcal C_{h+1,2i}$. We need to introduce three important steps related to node $(h,i)$:
\begin{itemize}
\item $\bar t_{h,i} :=\max_{t\in\mathcal C_{h,i}} t$ is the last time $(h,i)$ has been selected,
\item $\tilde t_{h,i} :=\max_{t\in\mathcal C^c_{h,i}} t$ is the last time when any of the two children of $(h,i)$ has been selected,
\item $t_{h,i} := \min\{t: T_{h,i}(t) > \tau_h(t)\}$ is the step when $(h,i)$ is expanded. 
\end{itemize} 

\textbf{The choice of $\tau_h$.} The threshold on the the number of pulls needed before expanding a node at depth $h$ is determined so that, at each time $t$, the two confidence terms in the definition of $U$ (Eq.~\ref{eq:def.stats.up}) are roughly equivalent, that is
\begin{align*}
\nu_1 \rho^h = c\sqrt{\frac{\log(1/\tilde \delta(t^+))}{\tau_h(t)}} \quad\Longrightarrow\quad \tau_h(t) = \frac{c^2 \log(1/\tilde \delta(t^+))}{\nu_1^2} \rho^{-2h}.
\end{align*}
Furthermore, since $t \leq  t^+ \leq 2t$ then
\begin{equation}\label{eq:tau2}
\frac{c^2}{\nu_1^2} \rho^{-2h} \leq \frac{c^2 \log(1/\tilde \delta(t))}{\nu_1^2} \rho^{-2h} \leq \tau_h(t) \leq  \frac{c^2 \log(2/\tilde \delta(t))}{\nu_1^2} \rho^{-2h},
\end{equation}
where we used the fact that $0 < \tilde \delta(t) \leq 1$ for all $t>0$. As described in Section~\ref{s:algorithm}, the idea is that the expansion of a node, which corresponds to an increase in the resolution of the approximation of $f$, should not be performed until the empirical estimate $\hmu_{h,i}$ of $f(x_{h,i})$ is accurate enough.
Notice that the number of pulls $T_{h,i}(t)$ for an expanded node $(h,i)$ does not necessarily coincide with $\tau_h(t)$, since $t$ might correspond to a time step when some leaves have not been pulled until $\tau_h(t)$ and other nodes have not been fully resampled after a refresh phase.


We begin our analysis by bounding the maximum depth of the trees constructed by \HCTiid.

\textbf{Lemma~\ref{lem:bound.depth.anytime}}
\textit{Given the number of samples $\tau_h(t)$ required for the expansion of nodes at depth $h$ in Eq.~\ref{eq:tau}, the depth $H(n)$ of the tree $\T_n$ is bounded as}
\begin{align*}
H(n) \leq H_{\max}(n) = \frac{1}{1-\rho}\log\Big(\frac{n\nu_1^2}{2(c\rho)^2}\Big).
\end{align*}

\begin{proof}
The deepest tree that can be developed by \HCTiid is a \textit{linear} tree, where at each depth $h$ only one node is expanded, that is ,  $|\calI^+_h(n)|=1$ and $|\calI_h(n)|=2$ for all $h<H(n)$. Thus we have
\begin{align*}
n & = \sum_{h=0}^{H(n)} \sum_{i\in \calI_h(n)} T_{h,i}(n) \geq \sum_{h=0}^{H(n)-1} \sum_{i\in \calI_h(n)} T_{h,i}(n) \geq \sum_{h=0}^{H(n)-1} \sum_{i\in \calI^+_h(n)} T_{h,i}(n) \geq \sum_{h=0}^{H(n)-1} \sum_{i\in \calI^+_h(n)} T_{h,i}(t_{h,i}) \\
&\stackrel{(1)}{\geq}  \sum_{h=0}^{H(n)-1} \sum_{i\in \calI^+_h(n)} \tau_{h,i}(t_{h,i})\geq \sum_{h=1}^{H(n)-1} \frac{c^2}{\nu_1^2} \rho^{-2h} \geq \frac{(c\rho)^2}{\nu_1^2} \rho^{-2H(n)} \sum_{h=1}^{H(n)-1} \rho^{-2(h-H(n)+1)},
\end{align*}
where inequality $(1)$ follows from the fact that a node $(h,i)$ is expanded at time $t_{h,i}$ only when it is pulled \textit{enough}, i.e., $T_{h,i}(t_{h,i}) \geq \tau_h(t_{h,i})$.
Since all the elements in the summation over $h$ are positive, then we can lower-bound the sum by its last element ($h=H(n)$), which is $1$, and obtain
\begin{align*}
n \geq 2\frac{(c\rho)^2}{\nu_1^2} H(n)\rho^{-2H(n)} \geq 2\frac{(c\rho)^2}{\nu_1^2} \rho^{-2H(n)},
\end{align*}
where we used the fact that $H(n) \geq 1$. By solving the previous expression we obtain
\begin{align*}
\rho^{-2H(n)} \leq n\frac{\nu_1^2}{2(c\rho)^2} \quad\Longrightarrow\quad H(n) \leq \frac{1}{2}\log\Big(\frac{n\nu_1^2}{2(c\rho)^2}\Big)/\log(1/\rho).
\end{align*}
Finally, the statement follows using $\log(1/\rho) \geq 1-\rho$.
\end{proof}

We now introduce a high probability event under which the mean reward for all the expanded nodes is within a confidence interval of the empirical estimates at a \textit{fixed} time $t$.


\begin{lemma}[High-probability event]\label{lem:high.prob}
We define the set of all the possible nodes in trees of maximum depth $H_{\max}(t)$ as
\begin{align*}
\L_t = \bigcup_{\T: \text{Depth}(\T)\leq H_{\max}(t)} \text{Nodes}(\T).
\end{align*}
We introduce the event
\begin{align*}
\calE_t= \bigg\{\forall (h,i)\in\L_t, \forall T_{h,i}(t)=1..t: \Big| \hmu_{h,i}(t) - f(x_{h,i}) \Big| \leq c\sqrt{\frac{\log(1 /\tilde \delta(t))}{T_{h,i}(t)}} \bigg\},
\end{align*}
where $x_{h,i}\in\calP_{h,i}$ is the arm corresponding to node $(h,i)$. If 
\begin{align*}
c=2\sqrt{\frac{1}{1-\rho}} \quad \text{ and } \quad \tilde \delta(t)=\frac{\delta}{t}\sqrt[8]{\frac{\rho}{3\nu_1}},
\end{align*}
then for any fixed $t$, the event $\calE_t$ holds with probability at least $1 - \delta/t^6$.
\end{lemma}

\begin{proof}
We upper bound the probability of the complementary event as
\begin{align*}
\mathbb{P}[\calE^{\comp}_t]  &\leq \sum_{ (h,i) \in \L_t } \sum_{T_{h,i}(t)=1}^t \mathbb{P}\bigg[\big| \hmu_{h,i}(t) - \mu_{h,i} \big| \geq c\sqrt{\frac{\log(1/\tilde \delta(t))}{T_{h,i}(t)}}\bigg] \\
&\leq \sum_{(h,i)\in \L_t} \sum_{T_{h,i}(t)=1}^t 2\exp\bigg(-2T_{h,i}(t) c^2 \frac{\log(1/\tilde \delta(t))}{T_{h.i}(t)}\bigg) \\
&= 2 \exp\big(-2c^2 \log(1/\tilde \delta(t))\big)  t |\L_t|,
\end{align*}
where the first inequality is an application of a union bound and the second inequality follows from the Chernoff-Hoeffding inequality.
We upper bound the number of nodes in $\L_t$ by the largest binary tree with a maximum depth $H_{\max}(t)$, i.e., $|\L_t| \leq 2^{H_{\max}(t)+1}$. Thus
\begin{align*}
\mathbb{P}[\calE^\comp_t] &\leq 2 (\tilde \delta(t))^{2c^2 } t 2^{H_{\max}(t)+1}.
\end{align*}
We first derive a bound on the the term $2^{H_{\max}(t)}$ as
\begin{align*}
2^{H_{\max}(t)} \leq \text{pow}\Bigg(2, \log_2\bigg(\frac{t\nu_1^2}{2(c\rho)^2} \bigg)^{\frac{1}{2\log_2(e)(1-\rho)}}\Bigg) \leq \bigg(\frac{t\nu_1^2}{2(c\rho)^2} \bigg)^{\frac{1}{2(1-\rho)}},
\end{align*}
where we used the upper bound $H_{\max}(t)$ from Lemma~\ref{lem:bound.depth.anytime} and $\log_2(e) > 1$. This leads to 
\begin{align*}
\mathbb{P}[\calE^\comp_t] &\leq 4t \big(\tilde \delta(t)\big)^{2c^2} \bigg(\frac{t\nu_1^2}{2(c\rho)^2} \bigg)^{\frac 1{2(1-\rho)}}.
\end{align*}
The choice of  $c$ and $\tilde \delta(t)$ as in the statement leads to
\begin{align*}
\mathbb{P}[\calE^\comp_t] &\leq 4t \bigg(\sqrt[8]{\rho/(3\nu_1)}\delta/t\bigg)^{\frac{8}{1-\rho}} \bigg(\frac{t\nu_1^2 (1-\rho)}{8\rho^2} \bigg)^{\frac 1{2(1-\rho)}} \\
&= 4 t \big(\delta/t\big)^{\frac{8}{1-\rho}} \big(\rho/(3\nu_1)\big)^{\frac{1}{1-\rho}} t^{\frac{1}{2(1-\rho)}} \bigg(\frac{\nu_1 \sqrt{1-\rho}}{\sqrt{8}\rho} \bigg)^{\frac 1{1-\rho}} \\
&\leq 4 \delta t^{1-\frac{8}{1-\rho}+ \frac{1}{2(1-\rho)}} \bigg(\frac{\sqrt{1-\rho}}{3\sqrt{8}} \bigg)^{\frac 1{1-\rho}}\\
&\leq \frac{4 }{ 3\sqrt{8} } \delta t^{\frac{-2\rho-13}{2(1-\rho)}} \leq \delta t^{-13/2} \leq \delta/t^6 ,
\end{align*}
which completes the proof.
\end{proof}

Recalling the definition the regret from Sect.~{s:preliminaries}, we decompose the regret of \HCTiid in two terms depending on whether event $\calE_t$ holds or not (i.e., failing confidence intervals). Let the instantaneous regret be $\Delta_t = f^* - r_t$, then we rewrite the regret as
\begin{equation}\label{eq:regret.decomp.iid}
R_n=\sum_{t=1}^n \Delta_t = \sum_{t=1}^n \Delta_t \mathbb I_{\calE_t} + \sum_{t=1}^n\Delta_t \mathbb I_{\calE^\comp_t} = R_n^{\calE} + R_n^{\calE^\comp}.
\end{equation}
We first study the regret in the case of failing confidence intervals.


\begin{lemma}[Failing confidence intervals]\label{lem:failing.bound}
Given the parameters $c$ and $\tilde\delta(t)$ as in Lemma~\ref{lem:high.prob}, the regret of \HCTiid when confidence intervals fail to hold is bounded as
\begin{align*}
R_n^{\calE^\comp} \leq \sqrt{n},
\end{align*}
with probability $1-\frac {\delta}{5n^2}$.
\end{lemma}

\begin{proof}
We first split the time horizon $n$ in two phases: the first phase until $\sqrt{n}$ and the rest. Thus the regret becomes
\begin{equation*}
R_n^{\calE^\comp} = \sum_{t=1}^n\Delta_t \mathbb I_{\calE^\comp_t}= \sum_{t=1}^{\sqrt{n}}\Delta_t \mathbb I_{\calE^\comp_t}+\sum_{t=\sqrt{n}+1}^n\Delta_t \mathbb I_{\calE^\comp_t}.
\end{equation*} 
We trivially bound the regret of first term by $\sqrt{n}$. So in order to prove the result it suffices to show that  event $\calE^\comp_t$ never happens after $\sqrt{n}$, which implies that the remaining term is zero with high probability.
By summing up the probabilities $\mathbb{P}[\calE^\comp_t]$ from $\sqrt{n}+1$ to $n$  and applying union bound we deduce
\begin{equation*}
\mathbb{P}\bigg[\bigcup_{t=\sqrt{n}+1}^{n} \calE^\comp_t\bigg] \leq \sum_{t=\sqrt{n}+1}^n\mathbb{P}[\calE^\comp_t] \leq  \sum_{\sqrt{n}+1}^n  \frac {\delta}{t^6}  \leq \int_{\sqrt{n}}^{+\infty}  \frac {\delta}{t^6} dt\leq\frac {\delta}{5n^{5/2}} \leq\frac {\delta}{5n^{2}}.
\end{equation*}
In words this result implies that w.p. $\geq1-\delta/(5 n^2)$ we can not have a failing confidence interval after time $\sqrt{n}$. This   combined with the trivial bound of $\sqrt{n}$ for the first $\sqrt{n}$ steps completes the proof.
\end{proof}

We are now ready to prove the main theorem, which only requires to study the regret term under events $\{\calE_t\}$.


\textbf{Theorem~\ref{thm:hct.iid} (Regret bound of \HCTiid).} 
\textit{Let $\delta\in(0,1)$, $\tilde \delta(t)=\sqrt[8]{\rho/(3\nu_1)}\delta/t$, and $c=2\sqrt{1/(1-\rho)}$. We assume that Assumptions~\ref{asm:dissim}--\ref{asm:near.optimal.dim} hold and that at each step $t$, the reward $r_t$ is independent of all prior random events and $\mathbb E(r_t|x_{t})=f(x_{t})$. Then the regret of \HCTiid after $n$ steps is}
\begin{align*}
R_n \leq 3 \bigg(\frac{2^{2d+7} \nu_1^{2(d+1)} C \nu_2^{-d} \rho^d}{(1-\rho)^{d+7}}\bigg)^{\frac{1}{d+2}} \bigg(\log\Big(\frac{2n}{\delta} \sqrt[8]{\frac{3\nu_1}{\rho}}\Big)\bigg)^{\frac{1}{d+2}} n^{\frac{d+1}{d+2}} + 2\sqrt{n\log(4n/\delta)},
\end{align*}
\noindent\textit{with probability $1-\delta$.}

\begin{proof}

\textbf{Step 1: Decomposition of the regret.} We start by further decomposing the regret in two terms. We rewrite the instantaneous regret $\Delta_t$ as
\begin{align*}
\Delta_t = f^* - r_t = f^* - f(x_{h_t,i_t}) + f(x_{h_t,i_t}) - r_t = \Delta_{h_t,i_t} + \hDelta_t,
\end{align*} 
which leads to a regret (see Eq.~\ref{eq:regret.decomp.iid})
\begin{align}\label{eq:regret.optimistic.nodes.fast.iid}
R_n^{\calE} = \sum_{t=1}^n \Delta_{h_t,i_t} \I_{\calE_t} + \sum_{t=1}^n \hDelta_t \I_{\calE_t} \leq \sum_{t=1}^n \Delta_{h_t,i_t} \I_{\calE_t} + \sum_{t=1}^n \hDelta_t = \widetilde{R}_n^{\calE} + \widehat{R}_n^{\calE}.
\end{align}

We start bounding the second term. We notice that the sequence $\{\widehat \Delta_t\}_{t=1}^n$ is a bounded martingale difference sequence since $\mathbb E (\widehat \Delta_t|\mathcal F_{t-1} )=0$ and $|\hDelta_t| \leq 1$. Therefore, an immediate application of the Azuma's inequality leads to
\begin{equation}
\label{eq:mart.bound.iid}
\widehat{R}_n^{\calE} = \sum_{t=1}^{n}\widehat \Delta_t\leq 2\sqrt{n\log(4n/\delta)},
\end{equation}
with probability $1-\delta/(4n^2)$.

\textbf{Step 2: Preliminary bound on the regret of selected nodes and their parents.} 
We now proceed with the study of the first term $\widetilde{R}_n^{\calE}$, which refers to the regret of the selected nodes as measured by its mean-reward. 
We start by characterizing which nodes are actually selected by the algorithm under event $\calE_t$. Let $(h_t,i_t)$ be the node chosen at time $t$ and $P_t$ be the path from the root to the selected node.
Let $(h',i')\in P_t$ and $(h'',i'')$ be the node which immediately follows $(h',i')$ in $P_t$ (i.e., $h''=h'+1$). By definition of $B$ and $U$ values, we have that
\begin{align}\label{eq:b2b}
B_{h',i'}(t) \!=\! \min\Big[U_{h',i'}(t); \max\big(B_{h'+1,2i'-1}(t); B_{h'+1,2i'}(t)\big)\Big] \!\leq\! \max\big(B_{h'+1,2i'-1}(t); B_{h'+1,2i'}(t)\big) \!=\! B_{h'',i''}(t),
\end{align}
where the last equality follows from the fact that the \textsl{OptTraverse} function selects the node with the largest $B$ value.
By iterating the previous inequality for all the nodes in $P_t$ until the selected node $(h_t,i_t)$ and its parent  $(h^p_t,i^p_t)$, we obtain that 
\begin{align*}
B_{h',i'}(t) &\leq B_{h_t,i_t}(t) \leq U_{h_t,i_t}(t), \qquad \forall (h',i')\in P_t
\\
B_{h',i'}(t) &\leq B_{h^p_t,i^p_t}(t) \leq U_{h^p_t,i^p_t}(t), \qquad \forall (h',i')\in P_t-{(h_t,i_t)}
\end{align*}
by definition of $B$-values. 
Thus for any node $(h,i)\in P_t \}$, we have that $U_{h_t,i_t}(t) \geq B_{h,i}(t)$. Furthermore, since  the root node $(0,1)$ which covers the whole arm space $\mathcal X$   is in $P_t$, thus   there exists at least one node $(h^*,i^*)$ in the set $P_t$  which includes the maximizer $x^*$ (i.e., $x^* \in \calP_{h^*,i^*}$) and has the the depth  $h^*\leq h^p_t< h_t$.\footnote{Note that we never pull the  root node $(0,1)$, therefore $h_t>0$.} Thus 
\begin{equation}
\label{eq:uq}
\begin{aligned}
U_{h_t,i_t}(t) &\geq B_{h^*,i^*}(t).
\\
U_{h^p_t,i^p_t}(t)& \geq B_{h^*,i^*}(t)
\end{aligned}
\end{equation}
Notice that in the set $P_t$ we may have multiple nodes $(h^*,i^*)$ which contain $x^*$ and that for all of them we have the following sequence of inequalities holds
\begin{align}\label{eq:diam.bound}
f^*-f(x_{h^*,i^*}) \leq \ell(x^*, x_{h^*,i^*}) \leq \text{diam}(\calP_{h^*,i^*}) \leq \nu_1\rho^{h^*},
\end{align}
where the second inequality holds since $x^*\in \calP_{h^*,i^*}$.

Now we expand the inequality in Eq.~\ref{eq:uq} on both sides using the high-probability event $\calE_t$. First we have
\begin{align}
U_{h_t,i_t}(t) &= \hmu_{h_t,i_t}(t) + \nu_1\rho^{h_t} + c\sqrt{\frac{\log(1/\tilde \delta(t^+))}{T_{h_t,i_t}(t)}} \leq f(x_{h_t,i_t}) + c\sqrt{\frac{\log(1/\tilde \delta(t))}{T_{h_t,i_t}(t)}} + \nu_1\rho^{h_t} + c\sqrt{\frac{\log(1/\tilde \delta(t^+))}{T_{h_t,i_t}(t)}} \nonumber\\
&\leq f(x_{h_t,i_t}) + \nu_1\rho^{h_t} + 2c\sqrt{\frac{\log(1/\tilde \delta(t^+))}{T_{h_t,i_t}(t)}}, 
\label{eq:boundU.iid}
\end{align}
where the first inequality holds on $\calE$ by definition of $U$ and the second by the fact that $t^+ \geq t$ (and $\log(1/\tilde \delta(t)) \leq \log(1/\tilde \delta(t^+))$).
The same result also holds for $(h^p_t,i^p_t)$ at time $t$:
\begin{align}
U_{h^p_t,i^p_t}(t) & \leq f(x_{h^p_t,i^p_t}) + \nu_1\rho^{h^p_t} + 2c\sqrt{\frac{\log(1/\tilde \delta(t^+))}{T_{h^p_t,i^p_t}(t)}}. 
\label{eq:boundUP.iid}
\end{align}

We now show that for any node $(h^*,i^*)$ such that $x^*\in \calP_{h^*,i^*}$, then $U_{h^*,i^*}(t)$ is a valid upper bound  on $f^*$:
\begin{align*}
U_{h^*,i^*}(t)&=\hmu_{h^*,i^*}(t) + \nu_1\rho^h + c\sqrt{\frac{\log(1/\tilde \delta(t^+))}{T_{h^*,i^*}(t)}}\overset{(1)}{\geq}  \hmu_{h^*,i^*}(t) + \nu_1\rho^{h^*} + c\sqrt{\frac{\log(1/\tilde \delta(t) )}{T_{h^*,i^*}(t)}}\nonumber\\
& \overset{(2)}{\geq} f(x_{h^*,i^*}) + \nu_1\rho^{h^*}  \overset{(3)}{\geq} f^*,
\end{align*}
where (1) follows from the fact that $t^+\geq t$, on (2)    we rely on the fact that the event $\calE_t$ holds at time $t$ and on (3)  we use the regularity of the function w.r.t. the maximum $f^*$ from Eq.~\ref{eq:diam.bound}. If an optimal node $(h^*,i^*)$  is a leaf, then $B_{h^*,i^*}(t)=U_{h^*,i^*}(t) \geq f^*$.
In the case that $(h^*,i^*)$ is not a leaf, there always exists a leaf $(h^+,i^+)$  
 such that $x^*\in\calP_{h^+,i^+}$ for which $(h^*,i^*)$ is its ancestor, since all the optimal nodes with $h>h^*$ are descendants of $(h^*,i^*)$. Now by propagating the bound backward from $(h^+,i^+)$ to $(h^*,i^*)$ through Eq.~\ref{eq:def.stats.upB} (see Eq.~\ref{eq:b2b}) we can show that  $B_{h^*,i^*}(t)$ is still a valid upper bound of the optimal value $f^*$. Thus for any optimal node $(h^*,i^*)$ at time $t$  under the event $\calE_t$ we have

\begin{equation*}
 B_{h^*,i^*}(t) \geq  f^*. 
\end{equation*}

Combining  this with Eq. \ref{eq:boundU.iid}, Eq. \ref{eq:boundUP.iid} and  Eq.~\ref{eq:uq} , we obtain that on event $\calE_t$ the selected node $(h_t,i_t)$ and its parent $(h^p_t,i^p_t)$ at any time $t$ is such that 
\begin{equation}\label{eq:regret.selected.nodes}
\begin{aligned}
\Delta_{h_t,i_t} &= f^* - f(x_{h_t,i_t})\leq  \nu_1\rho^{h_t} + 2c\sqrt{\frac{\log(1/\tilde \delta(t^+))}{T_{h_t,i_t}(t)}}.
\\
\Delta_{h^p_t,i^p_t} &= f^* - f(x_{h^p_t,i^p_t})\leq  \nu_1\rho^{h^p_t} + 2c\sqrt{\frac{\log(1/\tilde \delta(t^+))}{T_{h^p_t,i^p_t}(t)}}.
\end{aligned}
\end{equation}
Furthermore, since \HCTiid only selects nodes with $T_{h,i}(t)<\tau_h(t)$ the previous expression can be further simplified as
\begin{align}\label{eq:regret.optimistic.nodes.fast}
\Delta_{h_t,i_t} \leq 3c\sqrt{\frac{\log(2/\tilde \delta(t))}{T_{h_t,i_t}(t)}},
\end{align}
where we also used that $t^+\leq 2t$ for any $t$. Although this provides a preliminary bound on the instantaneous regret of the selected nodes, we need to further refine this bound. 

In the case of parent $(h^p_t,i^p_t)$, since  $T_{h^p_t,i^p_t}(t)\geq\tau_{h^p_t}(t)$, we deduce

\begin{align}\label{eq:regret.optimistic.nodes.fastP}
\Delta_{h^p_t,i^p_t} \leq  \nu_1\rho^{h^p_t} + 2c\sqrt{\frac{\log(1/\tilde \delta(t^+))}{\tau_{h^p_t}(t)}}=3\nu_1\rho^{h^p_t},
\end{align}

This implies that every selected node $(h_t,i_t)$ has a $3\nu_1 \rho^{h_{t}-1}$-optimal parent under the event $\calE_t$.

\textbf{Step 3: Bound on the cumulative regret.} 
We first decompose $\widetilde R_n^\calE$ over different depths. Let $1\leq \barH\leq H(n)$ a constant to be chosen later, then we have
\begin{equation}
\label{eq:tildeR.hct.iid} 
\begin{aligned}
\widetilde{R}_n^{\calE} &=  \sum_{t=1}^n \Delta_{h_t,i_t}\mathbb I_{ \calE_t}  \leq \sum_{h=0}^{H(n)}\sum_{i \in \calI_h(n)} \sum_{t=1}^n \Delta_{h,i} \mathbb I_{ (h_t,i_t)=(h,i) }\mathbb I_{ \calE_t}
\\
&\overset{(1)}{\leq}\sum_{h=0}^{H(n)}\sum_{i \in \calI_h(n)} \sum_{t=1}^n3c\sqrt{\frac{\log(2/\tilde \delta(t))}{T_{h,i}(t)}}\mathbb I_{ (h_t,i_t)=(h,i) }\overset{(2)}{\leq}\sum_{h=0}^{H(n)}\sum_{i \in \calI_h(n)}  \sum_{s=1}^{T_{h,i}(n)}3c\sqrt{ \frac{ \log( 2 / \tilde \delta(   \bar t_{h,i} ) ) }{s} }
\\
&\leq \sum_{h=0}^{H(n)}\sum_{i \in \calI_h(n)}  \int_{1}^{T_{h,i}(n)}3c\sqrt{ \frac{ \log( 2 / \tilde \delta(  \bar t_{h,i} ) ) }{s} }ds
\leq\sum_{h=0}^{H(n)}\sum_{i \in \calI_h(n)}6c\sqrt{ T_{h,i}(n) \log( 2/ \tilde \delta(  \bar t_{h,i} ) )  }
\\
&=6c\underbrace{\sum_{h=0}^{\barH }\sum_{i\in\calI_h(n)}\sqrt{T_{h,i}(n)\log(2/\tilde \delta(\bar t_{h,i}))}}_{(a)}+6c\underbrace{\sum_{h=\barH+1}^{H(n) }\sum_{i\in\calI_h(n)}\sqrt{T_{h,i}(n)\log(2/\tilde \delta(\bar t_{h,i}))}}_{(b)} 
\end{aligned}
\end{equation}  
%

where in (1) we rely on the definition of event $\calE_t$ and Eq.~\ref{eq:regret.optimistic.nodes.fast} and in (2) we rely on the fact that at any time step $t$ when the algorithm pulls the arm $(h,i)$, $T_{h,i}$ is incremented by 1 and that by definition of $\bar t_{h,i}$ we have that $t\leq \bar t_{h,i}$ . We now bound the two terms in the RHS of Eq.~\ref{eq:tildeR.hct.iid}. We first simplify the first term as
\begin{align}
\label{eq:bound.SampSz.iid}
(a) &= \sum_{h=0}^{H(n)}\sum_{i\in\calI_h(n)}\sqrt{T_{h,i}(n)\log(2/\tilde \delta(\bar t_{h,i}))} \leq \sum_{h=0}^{\barH }\sum_{i\in\calI_h(n)}\sqrt{\tau_h(n)\log(2/\tilde \delta(n))} \nonumber\\
&= \sum_{h=0}^{\barH } |\calI_h(n)|\sqrt{\tau_h(n)\log(2/\tilde \delta(n))},
\end{align}
where the inequality follows from $T_{h,i}(n) \leq \tau_h(n)$ and $\bar t_{h,i} \leq n$. We now need to provide a bound on the number of nodes at each depth $h$. We first notice that since $\T$ is a binary tree, the number of nodes at depth $h$ is at most twice the number of nodes at depth $h-1$ that have been expanded (i.e., the parent nodes), i.e., $|\calI_h(n)| \leq 2|\calI_{h-1}^+(n)|$.  We also recall the result of Eq. \ref{eq:regret.optimistic.nodes.fastP} which guarantees that $(h_t^p,i^p_t)$, the parent of the selected  node $(h_t,i_t)$,  is $3\nu_1 \rho^{h_t-1}$ optimal, that is, HCT never selects a node $(h_t,i_t)$ unless its parent is $3\nu_1 \rho^{h_t-1}$ optimal.
%
 %
%
%
%
%
 From Asm. \ref{asm:near.optimal.dim} we have that the number of $3\nu_1 \rho^{h}$-optimal nodes is bounded by the covering number $\N(3\nu_1/\nu_2\eps,l,\eps)$ with $\eps=\nu_1 \rho^{h}$. Thus we obtain the bound 
\begin{align}\label{eq:good.parent}
|\calI_h(n)| \leq 2|\calI_{h-1}^+(n)| \leq  2C( \nu_2 \rho^{(h-1)})^{-d},
\end{align}
where $d$ is the near-optimality dimension of $f$ around $x^*$. This bound combined with Eq.~\ref{eq:bound.SampSz.iid} implies that
\begin{align}\label{eq:first.tildeR.iid}
(a) &\leq\sum_{h=0}^{\barH}2  C \nu_2^{-d} \rho^{-(h-1)d} \sqrt{\tau_h(n)\log(2/\tilde \delta(n))} \leq \sum_{h=0}^{\barH} 2 C \nu_2^{-d} \rho^{-(h-1)d} \sqrt{\frac{c^2\log(1/\tilde \delta(n^+))}{\nu_1^2} \rho^{-2h}\log(2/\tilde \delta(n))} \nonumber\\
&\leq 2 C \nu_2^{-d} \frac{c\log(2/\tilde \delta(n^+))}{\nu_1}\rho^d \sum_{h=0}^{\barH} \rho^{-h(d+1)} \leq 2 C \nu_2^{-d} \frac{c\log(2/\tilde \delta(n^+))}{\nu_1}\rho^d \frac{\rho^{-\barH(d+1)}}{1-\rho}.
\end{align}
 We now bound the second term of Eq.~\ref{eq:tildeR.hct.iid} as
\begin{equation}\label{eq:boundHtoHn.any.iid}
\begin{aligned}
(b)\overset{(1)}{\leq}  &
 \sqrt{ \sum_{h=\barH+1}^{H(n) }\sum_{i\in\calI_h(n)}\log(2/\delta(\bar t_{h,i}))}\sqrt{\sum_{h=\barH+1}^{H(n) }\sum_{i\in\calI_h(n)}T_{h,i}(n)}\overset{(2)}{\leq} \sqrt{\sum_{h=\barH+1}^{H(n) }\sum_{i\in\calI_h(n)}\log(2/\tilde \delta(\bar t_{h,i}))}\sqrt{n}
  \end{aligned}
  \end{equation}
where in (1) we make use of Cauchy-Schwarz inequality and in (2) we simply bound the total number of samples by $n$. We now focus on the summation in the first square root. We recall that we denote by $\tilde t_{h,i}$ the last time when any of the two children of node $(h,i)$ has been pulled.
Then we have the following sequence of inequalities.
 \begin{equation}
\label{eq:IhHbar1.any.iid1}
\begin{aligned}
n&=\sum_{h=0}^{H(n)}\sum_{i\in \calI_h(n)}   T_{h,i}(n) \geq\sum_{h=0}^{H(n)-1}\sum_{i\in \calI^+_h(n)}   T_{h,i}(n)\geq\sum_{h=0}^{H(n)-1}\sum_{i\in \calI^+_h(n)}   T_{h,i}(\tilde t_{h,i}) \overset{(1)}{\geq} \sum_{h=0}^{H(n)-1}\sum_{i\in \calI^+_h(n)} \tau_{h}( \tilde t_{h,i} ) \\
&
\geq\sum_{h=\barH}^{H(n)-1}\sum_{i\in \calI^+_h(n)} \tau_{h}(\tilde t_{h,i})
\geq\sum_{h=\barH}^{H(n)-1}\sum_{i\in \calI^+_h(n)} \frac{\rho^{-2h} c^2\log(1/\tilde \delta(\tilde t_{h,i}^+))}{\nu_1^2}
\\
&\geq \frac{c^2\rho^{-2\barH}}{\nu_1^2} \sum_{h=\barH}^{H(n)-1}\rho^{2(\barH-h)}\sum_{i\in \calI^+_h(n)} \log(1/\tilde \delta(\tilde t^+_{h,i}))
\overset{(2)}{\geq}\frac{c^2\rho^{-2\barH}}{\nu_1^2} \sum_{h=\barH}^{H(n)-1}\sum_{i\in \calI^+_h(n)} \log(1/\tilde \delta(\tilde t^+_{h,i}))),
\end{aligned}
\end{equation} 
where in (1) we rely on the fact that, at each time step t, \textit{HCT-iid} only selects a node when $T_{h,i}(t)\geq \tau _{h,i}(t)$ for its parent   and in (2) we used that $\rho^{2(\barH-h)}\geq1$ for all $h\geq \barH$. We notice that, by definition of $\tilde t_{h,i}$,  for any internal node $(h,i)$  $ \tilde t_{h,i} = \max_{}(\bar t_{h+1,2i-1},\bar t_{h+1,2i})$. We also notice that for any $t_1,t_2>0$ we have that $[\max(t_1,t_2)]^+=\max(t_1^+,t_2^+)$. This implies that
 \begin{equation}
\label{eq:IhHbar1.any.iid2}
\begin{aligned}
n&\geq\frac{c^2\rho^{-2\barH}}{\nu_1^2} \sum_{h=\barH}^{H(n)-1}\sum_{i\in \calI^+_h(n)} \log(1/\tilde \delta([\max(\bar t_{h+1,2i-1},\bar t_{h+1,2i})]^+))
\\
&\overset{(1)}{=}\frac{c^2\rho^{-2\barH}}{\nu_1^2} \sum_{h=\barH}^{H(n)-1}\sum_{i\in \calI^+_h(n)}\max( \log(1/\tilde \delta(\bar t_{h+1,2i-1}^+)),\log(1/\tilde \delta(\bar t_{h+1,2i-1}^+)))
\\
&\overset{(2)}{\geq}\frac{c^2\rho^{-2\barH}}{\nu_1^2} \sum_{h=\barH}^{H(n)-1}\sum_{i\in \calI^+_h(n) }\frac{\log(1/\tilde \delta(\bar t_{h+1,2i-1}^+))+\log(1/\tilde \delta(\bar t_{h+1,2i}^+))}{2}
\\
&\overset{(3)}{=}\frac{c^2\rho^{-2\barH}}{2\nu_1^2} \sum_{h'=\barH+1}^{H(n)}\sum_{i\in \calI^+_{h'-1}(n) }\log(1/\tilde \delta(\bar t_{h',2i-1}^+))+\log(1/\tilde \delta(\bar t_{h',2i}^+))\\
&\overset{(4)}{=}\frac{c^2\rho^{-2\barH}}{2\nu_1^2} \sum_{h'=\barH+1}^{H(n)}\sum_{i'\in \calI_{h'}(n) }\log(1/\tilde \delta(\bar t_{h',i'}^+)),
\end{aligned}
\end{equation}
where in $(1)$ we rely on the fact that,   for any $t>0$, $ \log(1/\tilde \delta(t))$ is an increasing function of $t$. Therefore we have that $ \log(1/\tilde \delta(\max(t_1,t_2)))=\max(\log(1/\tilde \delta(t_1)) , \log(1/\tilde \delta(t_2)) )$  for any $t_1,t_2>0$ . In (2) we   rely on the fact that the maximum of  some random variables is always larger than their average.   We introduce a new  variable  $h'=h+1$ to derive (3).  For proving  (4) we rely on the argument that,  for any $h>0$,  $\calI^+_{h}(n) $ covers all the internal nodes  at layer $h$. This implies  that the set of the children of    $\calI^+_{h}(n) $ covers  $\calI_{h+1}(n)$.  This combined with fact that the inner sum  in   (3) is essentially taken  on the set of    the children of   $\calI^+_{h'-1}(n) $ proves (4).

Inverting Eq. \ref{eq:IhHbar1.any.iid2} we have
\begin{equation}
\label{eq:ih.bound.any.iid3}
\sum_{h=\barH+1}^{H(n)}\sum_{i\in \calI_h(n)} \log(1/\tilde \delta(\bar t^{+}_{h,i}))\leq \frac{2\nu_1^2\rho^{2\barH}n}{c^2}.
\end{equation}
By plugging Eq. \ref{eq:ih.bound.any.iid3} into Eq. \ref{eq:boundHtoHn.any.iid} we deduce
\begin{equation*}
\begin{aligned}
(b)&\leq\sqrt{\sum_{h=\barH+1}^{H(n) }\sum_{i\in\calI_h}\log(2/\tilde \delta(\bar t^{+}_{h,i}))}\sqrt{n} \leq\sqrt{\sum_{h=\barH+1}^{H(n) }\sum_{i\in\calI_h}2\log(1/\tilde \delta(\bar t^{+}_{h,i}))}\sqrt{n}\\
&\leq \sqrt{  \frac{4\nu_1^2\rho^{2\barH}n}{c^2}}\sqrt{n}= \frac{2}{c}\nu_1 \rho^{\barH  }n.
\end{aligned}
\end{equation*}
This combined with Eq. \ref{eq:first.tildeR.iid} provides the following  bound on $\widetilde R_n$:
\begin{equation*}
\widetilde R_n^{\calE} \leq12\nu_1\left[\frac{Cc^2\nu_2^{-d} \rho^d \log(2/\tilde \delta(n))}{\nu_1^2(1-\rho)}\rho^{-\barH(d+1)}+\rho^{\barH}n\right].
\end{equation*}
We then choose $\barH$ to minimize the previous bound. Notably we equalize the two terms in the bound by choosing
\begin{align*}
\rho^{\barH} = \bigg(\frac{c^2 C \nu_2^{-d} \rho^d}{(1-\rho)\nu_1^2} \frac{\log(2/\tilde \delta(n))}{n}\bigg)^{\frac{1}{d+2}},
\end{align*}
which, once plugged into the previous regret bound, leads to 
\begin{align*}
\widetilde R_n^{\calE} \leq \frac{24\nu_1}{c} \bigg(\frac{c^2 C \nu_2^{-d} \rho^d}{(1-\rho)\nu_1^2}\bigg)^{\frac{1}{d+2}} \big(\log(2/\tilde \delta(n))\big)^{\frac{1}{d+2}} n^{\frac{d+1}{d+2}}.
\end{align*}
Using the values of $\tilde\delta(t)$ and $c$ defined in Lemma~\ref{lem:high.prob}, the previous expression becomes
\begin{align*}
\widetilde R_n^{\calE} \leq 3 \bigg(\frac{2^{2(d+3)} \nu_1^{2(d+1)} C \nu_2^{-d} \rho^d}{(1-\rho)^{d/2+3}}\bigg)^{\frac{1}{d+2}} \bigg(\log\Big(\frac{2n}{\delta} \sqrt[8]{\frac{3\nu_1}{\rho}}\Big)\bigg)^{\frac{1}{d+2}} n^{\frac{d+1}{d+2}}.
\end{align*}
This combined with  the regret bound of Eq. \ref{eq:mart.bound.iid}  and the result of Lem.~\ref{lem:failing.bound} and a union bound on all $n\in\{1,2,3,\dots\}$ proves the final result with a probability at least $1-\delta$. 

\end{proof}


\section{Correlated Bandit feedback}

We begin the analysis of  \HCTgamma by proving some useful concentration inequalities for non-iid random variables under the mixing assumptions of Sect. \ref{s:preliminaries}.

\subsection{Concentration Inequality for non-iid Episodic Random Variables}

In this section we extend the result in~\citep{AzarLB13a} and we derive a concentration inequality for averages of non-iid random variables grouped in episodes. In fact, given the structure of the \HCTgamma algorithm, the rewards observed from an arm $x$ are not necessarily consecutive but they are obtained over multiple episodes. This result is of independent interest, thus we first report it in its general form and we later apply it to \HCTgamma.

In \HCTgamma, once an arm is selected, it is pulled for a number of consecutive steps and many steps may pass before it is selected again. As a result, the rewards observed from one arm are obtained through a series of episodes. Given a fixed horizon $n$, let $K_n(x)$ be the total number of episodes when arm $x$ has been selected, we denote by $t_k(x)$, with $k=1,\ldots,K_n(x)$, the step when $k$-th episode of arm $x$ has started and by $v_k(x)$ the length of episode $k$. Finally, $T_n(x) = \sum_{k}^{K_n(x)} v_k(x)$ is the total number of samples from arm $x$. The objective is to study the concentration of the empirical mean built using all the samples
\begin{align*}
\hmu_n(x) = \frac{1}{T_n(x)} \sum_{k=1}^{K_n(x)} \sum_{t=t_k(x)}^{t_k(x)+v_k(x)} r_t(x),
\end{align*}
towards the mean-reward $f(x)$ of the arm. 
In order to simplify the notation, in the following we drop the dependency from $n$ and $x$ and we use $K$, $t_k$, and $v_k$. We first introduce two quantities. For any $t=1,\ldots,n$ and for any $k=1,\ldots,K$, we define
\begin{align*}
M_t^k(x) = \E\Big[\sum_{t'=t_k}^{t_k+v_k} r_{t'} \big| \F_{t}\Big],
\end{align*}
as the expectation of the sum of rewards within episode $k$, conditioned on the filtration $\F_t$ up to time $t$ (see definition in Section~\ref{s:preliminaries}),\footnote{Notice that the index $t$ of the filtration can be before, within, or after the $k$-th episode.} and the residual
\begin{align*}
\eps_t^k(x) = M_t^k(x) - M_{t-1}^k(x).
\end{align*}
We prove the following.

\begin{lemma}\label{lem:martingale.residual}
For any $x\in\X$, $k=1,\ldots,K$, and $t=1,\ldots,n$, $\eps_t^k(x)$ is a bounded martingale sequence difference, i.e., $\eps_t^k(x) \leq 2\Gamma+1$ and $\E[\eps_t^k(x) | \F_{t-1}] = 0$.
\end{lemma}

\begin{proof}
Given the definition of $M_t^k(x)$ we have that
\begin{align*}
\eps_t^k(x) &= M_t^k(x) - M_{t-1}^k(x) = \E\Big[\sum_{t'=t_k}^{t_k+v_k} r_{t'} \big| \F_{t}\Big] - \E\Big[\sum_{t'=t_k}^{t_k+v_k} r_{t'} \big| \F_{t-1}\Big] \\
&=\sum_{t'=t_k}^t r_{t'} + \E\Big[\sum_{t'=t+1}^{t_k+v_k} r_{t'} \big| \F_{t}\Big] - \sum_{t'=t_k}^{t-1} r_{t'} - \E\Big[\sum_{t'=t}^{t_k+v_k} r_{t'} \big| \F_{t-1}\Big]\\
&=r_t + \E\Big[\sum_{t'=t+1}^{t_k+v_k} r_{t'} \big| \F_{t}\Big] - \E\Big[\sum_{t'=t}^{t_k+v_k} r_{t'} \big| \F_{t-1}\Big]\\
&=r_t - f(x) + \E\Big[\sum_{t'=t+1}^{t_k+v_k} r_{t'} \big| \F_{t}\Big] - (t_k+v_k-t)f(x) + (t_k+v_k-t+1)f(x) - \E\Big[\sum_{t'=t}^{t_k+v_k} r_{t'} \big| \F_{t-1}\Big]\\
&\leq 1 + \Gamma + \Gamma.
\end{align*}
Since the previous inequality holds both ways, we obtain that $|\eps_t^k(x)| \leq 2\Gamma +1$. Furthermore, we have that
\begin{align*}
\E\big[\eps_t^k(x)|\F_{t-1}] &= \E\big[M_t^k(x) - M_{t-1}^k(x)|\F_{t-1}\big] \\
&=\E\bigg[r_t + \E\Big[\sum_{t'=t+1}^{t_k+v_k} r_{t'} \big| \F_{t}\Big] \bigg| \F_{t-1}\bigg] - \E\Big[\sum_{t'=t}^{t_k+v_k} r_{t'} \big| \F_{t-1}\Big] = 0.
\end{align*}
\end{proof}

We can now proceed to derive a high-probability concentration inequality for the average reward of each arm $x$.

\begin{lemma}
\label{lem:tail.weak.dep}
For any $x\in\X$ pulled $K(x)$ episodes, each of length $v_k(x)$, for a total number of $T(x)$ samples, we have that
\begin{align}
\bigg| \frac{1}{T(x)} \sum_{k=1}^{K(x)} \sum_{t=t_k}^{t_k+v_k} r_t - f(x)\bigg| \leq (2\Gamma+1)\sqrt{\frac{2\log(2/\delta)}{T(x)}} + \frac{K(x)\Gamma}{T(x)},
\end{align}
with probability $1-\delta$.
\end{lemma}

\begin{proof}
We first notice that for any episode $k$\footnote{We drop the dependency of $M$ on $x$.}
\begin{align*}
\sum_{t=t_k}^{t_k+v_k} r_t = M_{t_k+v_k}^k,
\end{align*}
since $M_{t_k+v_k}^k = \E\Big[\sum_{t'=t_k}^{t_k+v_k} r_{t'} \big| \F_{t_k+v_k}\Big]$ and the filtration completely determines all the rewards. We can further develop the previous expression using a telescopic expansion which allows us to rewrite the sum of the rewards as a sum of residuals $\eps_t^k$ as
\begin{align*}
\sum_{t=t_k}^{t_k+v_k} r_t &= M_{t_k+v_k}^k = M_{t_k+v_k}^k - M_{t_k+v_k-1}^k + M_{t_k+v_k-1}^k - M_{t_k+v_k-2}^k + M_{t_k+v_k-2}^k + \cdots - M_{t_k}^k + M_{t_k}^k\\
&=\eps_{t_k+v_k}^k + \eps_{t_k+v_k-1}^k + \cdots + \eps_{t_k+1}^k + M_{t_k}^k = \sum_{t=t_k+1}^{t_k+v_k} \eps_t^k + M_{t_k}^k.
\end{align*}
Thus we can proceed by bounding
\begin{align*}
\bigg| \sum_{k=1}^{K(x)} \Big(\sum_{t=t_k}^{t_k+v_k} r_t -  v_k f(x)\Big)\bigg| &\leq \bigg|\sum_{k=1}^{K(x)}\sum_{t=t_k+1}^{t_k+v_k} \eps_t^k\bigg| + \bigg|\sum_{k=1}^{K(x)} \Big(M_{t_k}^k - v_k f(x)\Big)\bigg|\\
&\leq \bigg|\sum_{k=1}^{K(x)}\sum_{t=t_k+1}^{t_k+v_k} \eps_t^k\bigg| + K(x)\Gamma.
\end{align*}
By Lem.~\ref{lem:martingale.residual} $\eps_t^k$ is a bounded martingale sequence difference, thus we can directly apply the Azuma's inequality and obtain that
\begin{align*}
\bigg|\sum_{k=1}^{K(x)}\sum_{t=t_k+1}^{t_k+v_k} \eps_t^k\bigg| \leq (2\Gamma+1)\sqrt{2T(x)\log(2/\delta)}.
\end{align*}
Grouping all the terms together and dividing by $T(x)$ leads to the statement.
\end{proof}


\subsection{Proof of Thm. \ref{thm:hct.corr} }
\label{app:proof2}

The notation needed in this section is the same as in Section~\ref{app:proof1}. We only need to restate the notation about the episodes from previous section to \HCTgamma. We denote by $K_{h,i}(n)$ the number of episodes for node $(h,i)$ up to time $n$, by $t_{h,i}(k)$ the step when episode $k$ is started, and by $v_{h,i}(k)$ the number of steps of episode $k$.


We first notice that Lemma~\ref{lem:bound.depth.anytime} holds unchanged also for \HCTgamma, thus bounding the maximum depth of an \HCT tree to $ H(n) \leq H_{\max}(n) = \frac{1}{1-\rho}\log\Big(\frac{n\nu_1^2}{2(c\rho)^2}\Big)$. We begin the main analysis by applying the result of Lem. \ref{lem:tail.weak.dep} to bound the estimation error of $\hmu_{h,i}(t)$ at each time step $t$.

\begin{lemma}
\label{eq:EgodtoHoff}
Under Assumptions~\ref{asm:ergod} and \ref{asm:mixing}, for any fixed node $(h,i)$ and step $t$, we have that
\begin{align*}
| \hmu_{h,i}(t) - f(x_{h,i}) | \leq (3\Gamma+1)\sqrt{2\frac{\log(5/\delta)}{T_{h,i}(t)}}+\frac{\Gamma\log(t)}{T_{h,i}(t)}.
\end{align*}
with probability $1-\delta$. Furthermore, the previous expression can be conveniently restated for any $0 < \eps \leq 1$ as

\begin{equation*}
\mathbb P( | \hmu_{h,i}(t) - f(x_{h,i}) | >  \epsilon)\leq 5t^{1/3}\exp\left(-\frac{T_{h,i}(t) \eps^2}{2(3\Gamma+1)^2}\right)
\end{equation*}
\end{lemma}

\begin{proof}
As a direct consequence of Lem. \ref{lem:tail.weak.dep} we have w.p. $1-\delta$,

\begin{equation*}
| \hmu_{h,i}(t) - f(x_{h,i}) | \leq (2\Gamma+1)\sqrt{\frac{2\log(2/\delta)}{T_{h,i}(t)}}+\frac{K_{h,i}(t)\Gamma}{T_{h,i}(t)},
\end{equation*}

where $K_{h,i}(t)$ is the number of episodes  in which we pull arm $x_{h,i}$. At each episode in which $x_{h,i}$ is selected, its number of pulls $T_{h,i}$ is doubled w.r.t. the previous episode, except for those episodes where the current time $s$ becomes larger than $s^+$, which triggers the termination of the episode. However since $s^+$ doubles whenever $s$ becomes larger than $s^+$, the total number of times when episodes are interrupted because of $s\geq s^+$ can be at maximum  $\log_2(t)$ withing a time horizon of $t$. This means that the total number of times an episode finishes without doubling $T_{h,i}(t)$ is bounded by $\log_2(t)$. Thus we have
\begin{equation*}
\begin{aligned}
T_{h,i}(t)&\geq\sum_{k=1}^{ K_{h,i}(t)-\log_2(t)-1}2^{k-1} \geq2^{K_{h,i}(t)-\log_2(t)-2},
\end{aligned}
\end{equation*}
where in the second inequality we simply keep the last term of the summation. Inverting the previous inequality we obtain that
\begin{equation*}
\begin{aligned}
K_{h,i}(t)\leq \log_2(4T_{h,i}(t))+\log_2(t),
\end{aligned}
\end{equation*}
which bounds the number of episodes w.r.t. the number of pulls and the time horizon $t$.
Combining this result with the high probability bound of Lem. \ref{lem:tail.weak.dep}, we obtain
\begin{equation*}
\begin{aligned}
| \hmu_{h,i}(t) - f(x_{h,i}) | &\leq (2\Gamma+1)\sqrt{\frac{2\log(2/\delta)}{T_{h,i}(t)}}+\Gamma\frac{\log_2(4T_{h,i}(t))}{T_{h,i}(t)}+\Gamma\frac{\log(t)}{T_{h,i}(t)},
\end{aligned}
\end{equation*}
with probability $1-\delta$. The statement of the Lemma is obtained by further simplifying the second term in the right hand side with the objective of achieving a more homogeneous expression. In particular, we have that
\begin{equation*}
\begin{aligned}
\log_2(4T_{h,i}(t)) = 2\log_2(2\sqrt{T_{h,i}(t)}) = 2(\log_2(\sqrt{T_{h,i}(t)})+1) \leq 2\sqrt{T_{h,i}(t)},
\end{aligned}
\end{equation*}
and
\begin{equation*}
\begin{aligned}
| \hmu_{h,i}(t) - f(x_{h,i}) | &\leq (2\Gamma+1)\sqrt{\frac{2\log(2/\delta)}{T_{h,i}(t)}}+\frac{2\Gamma\sqrt{T_{h,i}(t)}}{T_{h,i}(t)}+\frac{\Gamma\log(t)}{T_{h,i}(t)}\\
&\leq (3\Gamma+1)\sqrt{\frac{2\log(5/\delta)}{T_{h,i}(t)}}+\frac{\Gamma\log(t)}{T_{h,i}(t)}.
\end{aligned}
\end{equation*}
To prove the second statement we choose $\eps:=(3\Gamma+1)\sqrt{\frac{2\log(5/\delta)}{T_{h,i}(t)}}+\frac{\Gamma\log(t)}{T_{h,i}(t)}$ and we solve the previous expression w.r.t. $\delta$:

\begin{equation*}
\delta=5\exp\left[-\frac{T_{h,i}(t) ( \eps- \Gamma\log(t)/T_{h,i}(t))^2}{2(3\Gamma+1)^2}\right].
\end{equation*}

 The following sequence of inequalities then follows

\begin{equation*}
\begin{aligned}
\mathbb P( | \hmu_{h,i}(t) - f(x_{h,i}) | > \eps) &\leq \delta=5\exp\left[-\frac{T_{h,i}(t) ( \eps- \Gamma\log(t)/T_{h,i}(t))^2}{2(3\Gamma+1)^2}\right]\leq 5\exp\left[-\frac{T_{h,i}(t) ( \eps^2- 2\eps\Gamma\log(t)/T_{h,i}(t))}{2(3\Gamma+1)^2}\right]
\\
&\leq 5\exp\left[-\frac{T_{h,i}(t)  (\eps^2- 2\Gamma\log(t)/T_{h,i}(t))}{2(3\Gamma+1)^2}\right]= 5\exp\left[-\frac{T_{h,i}(t)\eps^2}{(3\Gamma+1)^2}+ \frac{2\Gamma\log(t)}{2(3\Gamma+1)^2}\right]
\\
&\leq5\exp\left[-\frac{T_{h,i}(t)\eps^2}{(3\Gamma+1)^2}+ \frac{2\Gamma\log(t)}{12\Gamma}\right]
 =5\exp\left[-\frac{T_{h,i}(t)\eps^2}{2(3\Gamma+1)^2}+\log(t^{1/6})\right],
\end{aligned}
\end{equation*}

which concludes the proof.
\end{proof}

The result of Lem. \ref{eq:EgodtoHoff} facilitates the adaption of the previous results of iid case to the case of correlated rewards, since this bound is similar to  those of standard tail's inequality such as Hoeffding and Azuma's inequality. Based on this result we can extend the results of previous section to the case of dependent arms.

We now introduce the  high probability event  $\calE_{t,n}$ under which the mean reward for all the selected nodes in the interval  $[t,n]$ is within a confidence interval of the empirical estimates at every time step in the interval. The event $\calE_{t,n}$ is needed to concentrate the sum of obtained rewards  around the sum of their corresponding arm means. Note that unlike the previous theorem where we could make use of a simple martingale argument to concentrate the rewards around their means, here the rewards are not unbiased samples of the arm means. Therefore, we need  a more advanced  technique than the Azuma's inequality  for concentration of measure.

\begin{lemma}[High-probability event]\label{lem:high.prob.gamma}
We define the set of all the possible nodes in trees of maximum depth $H_{\max}(t)$ as
\begin{align*}
\L_t = \bigcup_{\T: \text{Depth}(\T)\leq H_{\max}(t)} \text{Nodes}(\T).
\end{align*}
We introduce the event
\begin{align*}
\Omega_{t}= \bigg\{\forall (h,i)\in\L_{t}, \forall T_{h,i}(t)=1,\dots,t : \big| \hmu_{h,i}(t) - f(x_{h,i}) \big| \leq c\sqrt{\frac{\log( 1/\tilde \delta(t))}{T_{h,i}(t)}}\bigg\},
\end{align*}
where $x_{h,i}\in\calP_{h,i}$ is the arm corresponding to node $(h,i)$, and the event $\calE_{t,n}= \bigcap_{s=t}^{n} \Omega_{s}$.
If 
\begin{align*}
c=6(3\Gamma+1)\sqrt{\frac{1}{1-\rho}} \quad \text{ and } \quad \tilde \delta(t)=\frac{\delta}{t}\sqrt[9]{\frac{\rho}{4\nu_1}},
\end{align*}
then for any fixed $t$, the event $\Omega_t$ holds with probability $1-\delta/t^7$ and the joint event $\calE_{t,n}$ holds with probability at least $1 - \delta/(6t^6)$.
\end{lemma}

\begin{proof}

We upper bound the probability of complementary event of $\Omega_t$ after $t$ steps
\begin{align*}
\mathbb{P}[\Omega^\comp_{t}]  &=  \sum_{ (h,i) \in \L_{t} } \sum_{T_{h,i}(t)=1}^{t} \mathbb{P}\bigg[\big| \hmu_{h,i}(t) - f(x_{h,i}) \big| \geq c\sqrt{\frac{\log(1/\tilde \delta(t))}{T_{h,i}(t)}}\bigg] \\
&\leq \sum_{(h,i)\in \L_t} \sum_{T_{h,i}(t)=1}^t 5t^{1/3}\exp\bigg(-T_{h,i}(t) c^2 \frac{\log(1/\tilde \delta(t))}{(3\Gamma+1)^2T_{h,i}(t)}\bigg) \\
&\leq 5 \exp(-c^2 /(3\Gamma+1)^2\log(1/\tilde \delta(t)))  t^{4/3} |\L_t|,
\end{align*}
Similar to the proof of Lem.~\ref{lem:failing.bound}, we have that $|\L_t| \leq 2^{H_{\max}(t)+1}$. Thus
\begin{align*}
\mathbb{P}[\Omega^\comp_t] &\leq 5(\tilde \delta(t))^{(c/(3\Gamma+1))^2 } t^{4/3} 2^{H_{\max}(t)+1}.
\end{align*}
We first derive a bound on the the term $2^{H_{\max}(t)}$ as
\begin{align*}
2^{H_{\max}(t)} \leq \text{pow}\Bigg(2, \log_2\bigg(\frac{t\nu_1^2}{2(c\rho)^2} \bigg)^{\frac{1}{2\log_2(e)(1-\rho)}}\Bigg) \leq \bigg(\frac{t\nu_1^2}{2(c\rho)^2} \bigg)^{\frac{1}{2(1-\rho)}},
\end{align*}
where we used the definition of the upper bound $H_{\max}(t)$. 
%
%
which leads to 
\begin{align*}
\mathbb{P}[\Omega^\comp_t] &\leq 10 t^{4/3} \big(\tilde \delta(t)\big)^{(c/(3\Gamma+1))^2} \bigg(\frac{t\nu_1^2}{2(c\rho)^2} \bigg)^{\frac 1{2(1-\rho)}}.
\end{align*}

The choice of  $c$ and $\tilde\delta(t)$ as in the statement leads to $\mathbb{P}[\Omega^\comp_t] \leq  \frac {\delta}{t^7} $ (steps are similar to Lemma~\ref{lem:high.prob}) .

The bound on the joint event $\calE_{t,n}$ follows from a union bound as

\begin{equation*}
\mathbb P \left [\calE^\comp_{t,n} \right]=\mathbb P \Big[ \bigcup_{s=t} ^{n}\Omega^\comp_{s} \Big]\leq \sum_{s=t} ^{n}\mathbb  P (\Omega^\comp_{s} )\leq \int_{t}^{\infty}\frac {\delta}{s^7} ds = \frac{\delta}{6t^6}.
\end{equation*}
\end{proof}

Recalling the definition of regret from Sect.~\ref{s:preliminaries}, we decompose the regret of \HCTiid in two terms depending on whether event $\calE_t$ holds or not (i.e., failing confidence intervals). Let the instantaneous regret be $\Delta_t = f^* - r_t$, then we rewrite the regret as
\begin{equation}\label{eq:regret.decomp.gamma}
R_n=\sum_{t=1}^n \Delta_t = \sum_{t=1}^n \Delta_t \mathbb I_{\calE_t} + \sum_{t=1}^n\Delta_t \mathbb I_{\calE^\comp_t} = R_n^{\calE} + R_n^{\calE^\comp}.
\end{equation}
We first study the regret in the case of failing confidence intervals.


\begin{lemma}[Failing confidence intervals]\label{lem:failing.bound.gamma}
Given the parameters $c$ and $\tilde\delta(t)$ as in Lemma~\ref{lem:high.prob.gamma}, the regret of \HCTiid when confidence intervals fail to hold is bounded as
\begin{align*}
R_n^{\calE^\comp} \leq \sqrt{n},
\end{align*}
with probability $1-\frac {\delta}{30 n^2}$.
\end{lemma}

\begin{proof}
The proof is the same as in Lemma~\ref{lem:failing.bound} expect for the union bound which is applied to $\calE_{t,n}$ for $t=\sqrt{n},\ldots,n$.
\end{proof}

We are now ready to prove the main theorem, which only requires to study the regret term under events $\{\calE_{t,n}\}$.


\textbf{Theorem~\ref{thm:hct.corr} (Regret bound of \HCTgamma).} 
\textit{Let $\delta\in(0,1)$ and  
$c:=6(3\Gamma+1)\sqrt{1/(1-\rho)}$. We assume that Assumptions~\ref{asm:ergod}--\ref{asm:near.optimal.dim} hold and that rewards are generated according to the general model defined in Section~\ref{s:preliminaries}. Then the regret of \HCTgamma after $n$ steps is}
\begin{align*}
R_n \leq 2(3\sqrt{2}+4) \bigg(  \frac{c^2 C\nu_1^{-2} \nu_2^{-d} \rho^d}{1-\rho}\bigg)^{\frac{1}{d+2}} \bigg(\log\Big(\frac{2n}{\delta} \sqrt[9]{\frac{3\nu_1}{\rho}}\Big)\bigg)^{\frac{1}{d+2}} n^{\frac{d+1}{d+2}} + \sqrt{n},
\end{align*}
\noindent\textit{with probability $1-\delta$.}

\begin{proof}

The structure of the proof is exactly the same as in Thm.~\ref{thm:hct.iid}. Thus, here we report only the main differences in each step.

\textbf{Step 1: Decomposition of the regret.} We first decompose the regret in two terms. We rewrite the instantaneous regret $\Delta_t$ as
\begin{align*}
\Delta_t = f^* - r_t = f^* - f(x_{h_t,i_t}) + f(x_{h_t,i_t}) - r_t = \Delta_{h_t,i_t} + \hDelta_t,
\end{align*} 
which leads to a regret
\begin{align}
R_n^{\calE} = \sum_{t=1}^n \Delta_{h_t,i_t} \I_{\calE_{t,n}} + \sum_{t=1}^n \hDelta_t \I_{\calE_{t,n}} = \widetilde{R}_n^{\calE} + \widehat{R}_n^{\calE}.
\end{align}
Unlike in Thm.~\ref{thm:hct.iid}, the definition of $\widehat{R}_n^{\calE}$ still requires the event $\I_{\calE_{t,n}}$ and the sequence $\{\widehat \Delta_t\}_{t=1}^n$ is no longer a bounded martingale difference sequence. In fact, $\mathbb E (\widehat \Delta_t|\mathcal F_{t-1} ) \neq 0$ since the expected value of $r_t$ does not coincide with the mean-reward value of the corresponding node $f(x_{h_t,i_t})$. This prevents from directly using the Azuma inequality and extra care is needed to derive a bound. We have that
\begin{equation}
\label{eq:tildeR.local} 
\begin{aligned}
\widehat{R}_n^{\calE} &=  \sum_{t=1}^n \widehat \Delta_t \I_{\calE_{t,n}} \leq \sum_{h=0}^{H(n)}\sum_{i \in \calI_h(n)} \sum_{t=1}^n  \widehat \Delta_t \I_{\calE_{t,n}} \mathbb I_{ (h_t,i_t)=(h,i) }
\\
&= \sum_{h=0}^{H(n)}\sum_{i \in \calI_h(n)} \sum_{t=1}^n  (f(x_{h,i}) - r_t ) \I_{\calE_{t,n}} \mathbb I_{ (h_t,i_t)=(h,i) }
\stackrel{(1)}{\leq}\sum_{h=0}^{H(n)}\sum_{i \in \calI_h(n)} \sum_{t=1}^n  (f(x_{h,i}) - r_t ) \I_{\Omega_{t_{h,i},n}} \I_{ (h_t,i_t)=(h,i) }
\\
&\stackrel{(2)}{=}
 \sum_{h=0}^{H(n)}\sum_{i \in \calI_h(n)} T_{h,i}(\bar t_{h,i})(f(x_{h,i}) -\hmu_{h,i}(\bar t_{h,i}) )\I_{\Omega_{\bar t_{h,i}}}
\\
& \stackrel{(3)}{\leq} \sum_{h=0}^{H(n)}\sum_{i \in \calI_h(n)} cT_{h,i}(\bar t_{h,i})\sqrt{\frac{\log(2/\tilde \delta(\bar t_{h,i}))}{T_{h,i}(\bar t_{h,i})}} \leq
\sum_{h=0}^{H(n)}\sum_{i \in \calI_h(n)} c\sqrt{T_{h,i}(\bar t_{h,i})\log(2/\tilde \delta(\bar t_{h,i}))}
\\
&\leq c\underbrace{\sum_{h=0}^{\barH }\sum_{i\in\calI_h(n)}\sqrt{T_{h,i}(n)\log(2/\tilde \delta(\bar t_{h,i}))}}_{(a)}+c\underbrace{\sum_{h=\barH+1}^{H(n) }\sum_{i\in\calI_h(n)}\sqrt{T_{h,i}(n)\log(2/\tilde \delta(\bar t_{h,i}))}}_{(b)},
\end{aligned}
\end{equation}  
where (1) follows from the definition of $\calE_{t,n} = \bigcap_{s=t}^n \Omega_s$, thus if $\calE_{t,n}$ holds at time $t$ then $\Omega_s$ also holds at $s=\bar t_{h,i}\geq t$. Step (2) follows from the definition of   $\hmu_{h,i}$:  First we  notice that  for the node $(h_n,i_n)$  we have that $T_{h_n,i_n}(n)\hmu_{h_n,i_n}(n)= \sum_{t=1}^n   r_t \I_{ (h_t,i_t)=(h_n,i_n) }$ since we update the statistics at the end.  for  every other  node we have that the last selection  time $\bar t_{h,i}$ and the end of last episode coincides together .    Now since we update the statistics of  the selected node   at the end of every episode, thus,  we have that $T_{h,i}(\bar t_{h,i})\hmu_{h,i}(\bar t_{h,i})= \sum_{t=1}^n   r_t \I_{ (h_t,i_t)=(h,i) }$ also for $(h,i)\neq (h_n,i_n)$.     Step (3) follows from the definition of $\Omega_s$. The resulting bound matches the one in Eq.~\ref{eq:tildeR.hct.iid} up to constants and it can be bound similarly.
\begin{equation*}
\widehat{R}_n^{\calE} \leq2\nu_1\left[\frac{Cc^2\nu_2^{-d} \rho^d \log(2/\tilde \delta(n))}{\nu_1^2(1-\rho)}\rho^{-\barH(d+1)}+\rho^{\barH}n\right].
\end{equation*}

\textbf{Step 2: Preliminary bound on the regret of selected nodes.} The second step follows exactly the same steps as in the proof of Thm.~\ref{thm:hct.iid} with the only difference that here we use the high-probability event $\mathcal E_{t,n}$. As a result the following inequalities hold   for the node $(h_t,i_t)$ selected at time $t$ and  its parent  $(h_t^p,i_t^p)$
\begin{equation}
\label{eq:regret.optimistic.nodes.fast.gamma}
\begin{aligned}
\Delta_{h_t,i_t} & \leq 3c\sqrt{\frac{\log(2/\tilde \delta( t))}{T_{h_t,i_t}( t )}}.
\\
\Delta_{h^p_t,i^p_t} & \leq 3\nu_1 \rho^{h_t-1}.
\end{aligned}
\end{equation}

\textbf{Step 3: Bound on the cumulative regret.} Unlike in the proof of Thm.~\ref{thm:hct.iid}, the total regret $\widetilde R_n^{\calE}$ should be analyzed with extra care since here  we  do not update  the selected arm  as well as the  statistics  $T_{h,i}(t)$ and $\widehat \mu_{h,i}(t)$ for the the  entire length of episode, whereas in Thm.~\ref{thm:hct.iid} we update at every step.
Thus the development of $\widetilde R_n^{\calE}$ slightly differs from Eq.~\ref{eq:tildeR.hct.iid}. Let $1\leq \barH\leq H(n)$ a constant to be chosen later, then we have
\begin{align}
\label{eq:tildeR.corr} 
\widetilde R_n^{\calE} & \stackrel{(1)}{=}  \sum_{t=1}^n \Delta_{h_t,i_t} \mathbb I_{\calE_{t,n}}  = \sum_{h=0}^{H(n)}\sum_{i \in \calI_h(n)} \sum_{t=1}^n \Delta_{h,i}\I_{ (h_t,i_t)=(h,i) }\mathbb I_{ \calE_{t,n} }  = \sum_{h=0}^{H(n)}\sum_{i \in \calI_h(n)} \sum_{k=1}^{K_{h,i}(n)} \sum_{t=t_{h,i}(k)}^{t_{h,i}(k)+v_{h,i}(k)} \Delta_{h,i} \mathbb I_{ \calE_{t,n} } 
 \nonumber\\
&
\stackrel{(2)}{\leq}\sum_{h=0}^{H(n)}\sum_{i \in \calI_h(n)} \sum_{k=1}^{K_{h,i}(n)}  \sum_{t=t_{h,i}(k)}^{t_{h,i}(k)+v_{h,i}(k)}\left[3c\sqrt{\frac{\log(2/\tilde \delta(t))}{T_{h,i}(t)}}\right] 
\stackrel{(3)}{=}\sum_{h=0}^{H(n)}\sum_{i \in \calI_h(n)} \sum_{k=1}^{K_{h,i}(n)} v_{h,i}(k)\left[3c\sqrt{\frac{\log(2/\tilde \delta(t_{h,i}(k)))}{T_{h,i}(t_{h,i}(k))}}\right] 
 \nonumber \\
&
\leq \sum_{h=0}^{H(n)}\sum_{i \in \calI_h(n)} 3c\sqrt{\log(2/\tilde \delta(\bar t_{h,i}))}\sum_{k=1}^{K_{h,i}(n)} \frac{v_{h,i}(k)}{\sqrt{T_{h,i}(t_{h,i}(k))}}
\nonumber\\
&
 \stackrel{(4)}{\leq} 3(\sqrt{2}+1)c\sum_{h=0}^{H(n)}\sum_{i \in \calI_h(n)} \sqrt{\log(2/\tilde \delta( \bar t_{h,i}))T_{h,i}(t_{h,i}(K_{h,i}(n)))}  
\leq3(\sqrt{2}+1)c\sum_{h=0}^{H(n)}\sum_{i \in \calI_h(n)} \sqrt{\log(2/\tilde \delta( \bar t_{h,i}))T_{h,i}(n)}  
\nonumber\\
&=3(\sqrt{2}+1)c\underbrace{\sum_{h=0}^{\barH }\sum_{i\in\calI_h(n)}\sqrt{T_{h,i}(n)\log(2/\tilde \delta(\bar t_{h,i}))}}_{(a)}+3(\sqrt{2}+1)c\underbrace{\sum_{h=\barH+1}^{H(n) }\sum_{i\in\calI_h(n)}\sqrt{T_{h,i}(n)\log(2/\tilde \delta(\bar t_{h,i}))}}_{(b)},
\end{align}
where the first sequence of equalities in (1) simply follows from the definition of episodes. In (2) we bound the instantaneous regret by Eq.~\ref{eq:regret.optimistic.nodes.fast.gamma}. 
Step (3) follows from  the fact that when $(h,i)$ is selected, its statistics, including  $T_{h,i}$, are not changed until the end of  the episode. Step (4) is an immediate application of Lemma~19 in~\cite{jaksch2010near-optimal}.

Constants apart the terms $(a)$ and $(b)$ coincides with the terms defined in Eq.~\ref{eq:tildeR.hct.iid} and similar bounds can be derived.

Putting the bounds on $\widehat R_n^{\calE}$ and $\widetilde R_n^{\calE}$ together leads to 
\begin{equation*}
R_n^{\calE} \leq2(3\sqrt{2}+4)\nu_1\left[\frac{Cc^2\nu_2^{-d} \rho^d \log(2/\tilde \delta(n))}{\nu_1^2(1-\rho)}\rho^{-\barH(d+1)}+\rho^{\barH}n\right].
\end{equation*}
It is not difficult to prove that for a suitable choice $\barH$, we obtain the final bound of $O(\log(n)^{1/(d+2)}n^{(d+1)/(d+2)})$ on $R_n$. This combined with the result of Lem.~\ref{lem:high.prob.gamma} and a union bound on all $n\in\{1,2,3,\dots\}$ proves the final result. 
 
\end{proof}


\subsection{Proof of Thm.~\ref{thm:hct.corr.space}}
\label{appx:hct.corr.space}

\textbf{Theorem~\ref{thm:hct.corr.space}}
\textit{Let $\delta\in(0,1)$, $\tilde \delta(n)=\sqrt[9]{\rho/(4\nu_1)}\delta/n$, and $c=3(3\Gamma+1)\sqrt{1/(1-\rho)}$. We assume that Assumptions~\ref{asm:ergod}--\ref{asm:near.optimal.dim} hold and that rewards are generated according to the general model defined in Section~\ref{s:preliminaries}. Then if $\delta = 1/n$ the space complexity of \HCTgamma is}
\begin{equation*}
\mathbb E(\N_n) =O( \log(n)^{2/(d+2)}n^{d/(d+2)} ).
\end{equation*}

\begin{proof}
We assume that the space requirement for each node (i.e., storing variables such as $\hmu_{h,i}$, $T_{h.i}$) is a unit. Let $\B_t$ denote the event corresponding to the branching/expansion of the node $(h_t,i_t)$ selected at time $t$, then the space complexity is $\N_n = \sum_{t=1}^n \I_{B_t}$. Similar to the regret analysis, we decompose $\N_n$ depending on events $\calE_{t,n}$, that is
\begin{align}
\N_n = \sum_{t=1}^n \I_{\B_t}\I_{\calE_{t,n}} + \sum_{t=1}^n \I_{\B_t}\I_{\calE_{t,n}^{\comp}} = \N_n^{\calE} + \N_n^{\calE^{\comp}}.
\end{align}
Since we are targeting the expected space complexity, we take the expectation of the previous expression and the second term can be easily bounded as
\begin{align}
\E\big[\N_n^{\calE^{\comp}}\big] = \sum_{t=1}^n \I_{\B_t}\Prob[\calE_{t,n}^{\comp}] \leq \sum_{t=1}^n \Prob[\calE_t^{\comp}] \leq \sum_{t=1}^n\frac{\delta}{6t^6} \leq C,
\end{align}
where the last inequality follows from Lemma~\ref{lem:high.prob.gamma} and $C$ is a constant independent from $n$. We now focus on the first term $\N_n^{\calE}$. We first rewrite it as the total number of nodes $|\T_n|$ generated by \HCT over $n$ steps. For any depth $\barH>0$ we have
\begin{align}\label{eq:high.prob.complexity}
\N_n^{\calE} = \sum_{h=0}^{H(n)} |\calI_h(n)| = 1+ \sum_{h=1}^{\barH} |\calI_h(n)| + \sum_{h=\barH+1}^{H(n)} |\calI_h(n)| \leq 1 + \underbrace{\barH |\calI_{\barH}(n)|}_{(c)} + \underbrace{\sum_{h=\barH+1}^{H(n)} |\calI_h(n)|}_{(d)}.
\end{align}
A bound on term (d) can be recovered through the following sequence of inequalities
 \begin{equation}
\label{eq:IhHbar.any}
\begin{aligned}
n &= \sum_{h=0}^{H(n)} \sum_{i\in\calI_h(n)} T_{h,i}(n) \geq \sum_{h=0}^{H(n)} \sum_{i\in\calI_h^+(n)} T_{h,i}(n) \stackrel{(1)}{\geq} \sum_{h=0}^{H(n)} \sum_{i\in\calI_h^+(n)} \tau_{h,i}(t_{h,i}) \\
&\stackrel{(2)}{\geq} \sum_{h=0}^{H(n)} \sum_{i\in\calI_h^+(n)} \frac{c^2}{\nu_1^2} \rho^{-2h} \stackrel{(3)}{\geq} \frac{1}{\nu_1^2}\sum_{h=\barH}^{H(n)-1} |\calI_h^+(n)|  \rho^{-2h} = \frac{1}{\nu_1^2} \rho^{-2\barH}\sum_{h=\barH}^{H(n)-1} |\calI_h^+(n)|  \rho^{2(\barH-h)} \\
&\geq \frac{1}{\nu_1^2} \rho^{-2\barH}\sum_{h=\barH}^{H(n)-1} |\calI_h^+(n)| \stackrel{(4)}{\geq} \frac{1}{2\nu_1^2} \rho^{-2\barH}\sum_{h=\barH+1}^{H(n)} |\calI_h(n)|,
\end{aligned}
\end{equation} 
where (1) follows from the fact that nodes in $\calI_h^+(n)$ have been expanded at time $t_{h,i}$ when their number of pulls $T_{h,i}(t_{h,i}) \leq T_{h,i}(n)$ exceeded the threshold $\tau_{h,i}(t_{h,i})$. Step (2) follows from Eq.~\ref{eq:tau2}, while (3) from the definition of $c>1$. Finally, step (4) follows from the fact that the number of nodes at depth $h$ cannot be larger than twice the parent nodes at depth $h-1$.
By inverting the previous inequality, we obtain
\begin{align*}
(d) \leq 2\nu_1^2 n \rho^{2\barH}.
\end{align*}
On other hand, in order to bound (c), we need to use the same the high-probability events $\calE_{t,n}$ and similar passages as in Eq.~\ref{eq:good.parent}, which leads to $|\calI_h(n)| \leq 2|\calI_{h-1}^+(n)| \leq  2C( \nu_2 \rho^{(h-1)})^{-d}$. Plugging these results back in Eq.~\ref{eq:high.prob.complexity} leads to 
\begin{align*}
\N_n^{\calE} \leq 1+ 2\barH C( \nu_2 \rho^{(\barH-1)})^{-d} + 2\nu_1^2 n \rho^{2\barH},
\end{align*}
with high probability. Together with $\N_n^{\calE^{\comp}}$ we obtain
\begin{align*}
\E\big[\N_n\big] \leq 1+ 2\barH C( \nu_2 \rho^{(\barH-1)})^{-d} + 2\nu_1^2 n \rho^{2\barH} + C \leq 1+ 2H_{\max}(n) C( \nu_2 \rho^{(\barH-1)})^{-d} + 2\nu_1^2 n \rho^{2\barH} + C,
\end{align*}
where $H_{\max}(n)$ is the upper bound on the depth of the tree in Lemma~\ref{lem:bound.depth.anytime}. Optimizing $\barH$ in the remaining terms leads to the statement.
\end{proof}

\bibliography{refs}
\bibliographystyle{icml2014}

\end{document}